\documentclass[10pt,twocolumn,letterpaper]{article}
\usepackage{cvpr}
\usepackage{times}
\usepackage{epsfig}
\usepackage{graphicx}
\usepackage{amsmath}
\usepackage{amssymb}

\usepackage{amsthm}
\usepackage{thmtools}
\usepackage{etex}
\usepackage{amsxtra}
\usepackage{amsfonts}
\usepackage{nicefrac}
\usepackage{microtype}
\usepackage{url}
\usepackage{soul}

\usepackage{color}
\usepackage{bbm}
\usepackage{pgfplots}
\usepackage{tikz}
\usetikzlibrary{datavisualization}

\usepackage{colortbl}
\newcolumntype{g}{>{\columncolor[gray]{0.85}}c}

\usepackage{enumitem}

\usepackage{wrapfig}
\usepackage{sidecap}
\usepackage{floatrow}
\usepackage{standalone}
\usepackage[ruled,linesnumbered]{algorithm2e}

\usepackage{stmaryrd} % for Iverson brackets
\usepackage{mathtools}

\usepackage{multirow}
\usepackage{fixmath}
\usepackage{booktabs}
\usepackage{pbox}
\usepackage{subfigure}
\usepackage[pagebackref=true,breaklinks=true,letterpaper=true,colorlinks,bookmarks=false]{hyperref}

\newcommand{\SM}{\mathcal{M}}

\newcommand{\SN}{\mathcal{N}}
\newcommand{\SG}{\mathcal{G}}
\newcommand{\SL}{\mathcal{L}}
\newcommand{\SX}{X}

\newcommand{\BF}{\mathbb{F}}

\newcommand{\BV}{\mathbb{V}}
\newcommand{\BE}{\mathbb{E}}
\newcommand{\BG}{\mathbb{G}}

\newcommand{\FG}{\mathsf{G}}
\newcommand{\FV}{\mathsf{V}}
\newcommand{\FE}{\mathsf{E}}
\newcommand{\FT}{\mathsf{T}}

\newcommand{\FL}{\mathsf{L}}

\newcommand{\IPSLP}{IRPS-LP}

\newcommand{\R}{\mathbb{R}}
\newcommand{\N}{\mathbb{N}}

\newcommand{\la}{\langle}
\newcommand{\ra}{\rangle}
\newcommand{\T}{\top}

\newcommand{\fdots}{\makebox[0.75em][c]{.\hfill.}\thinspace} % two dots to save space

\DeclareMathOperator*{\argmin}{arg min}
\DeclareMathOperator*{\argmax}{arg max}
\DeclareMathOperator*{\conv}{conv}

\providecommand{\abs}[1]{\lvert#1\rvert}

\definecolor{fettrot}{RGB}{255,10,10}

\newtheorem{definition}{Definition}

\newtheorem{lemma}{Lemma}
\newtheorem*{lemma*}{Lemma}
\newtheorem{proposition}{Proposition}
\newtheorem*{proposition*}{Proposition}
\newtheorem*{theorem*}{Theorem}

\declaretheoremstyle[bodyfont=\normalfont]{normalbody}
\declaretheorem[style=normalbody,name=Example]{example}
\newcommand{\myparagraph}[1]{\paragraph{#1}}

\makeatletter
\def\smallunderbrace#1{\mathop{\vtop{\m@th\ialign{##\crcr
   $\hfil\displaystyle{#1}\hfil$\crcr
   \noalign{\kern3\p@\nointerlineskip}%
   \tiny\upbracefill\crcr\noalign{\kern3\p@}}}}\limits}
\makeatother
\usepackage{cite}
\usepackage{longtable}
\cvprfinalcopy % *** Uncomment this line for the final submission

\setitemize{noitemsep,topsep=0pt,parsep=0pt,partopsep=0pt,leftmargin=*}
\makeatletter
\renewcommand{\paragraph}{%
  \@startsection{paragraph}{4}%
  {\z@}{1.25ex \@plus 1ex \@minus .2ex}{-1.9em}%
  {\normalfont\normalsize\bfseries}%
}
\makeatother

\begin{document}
% \linenumbers

% can we make the title stronger? 

%\title{Novel Message Passing Algorithms for Graph Matching}
% Dual ascent Algorithms for graph matching
% Study of Lagrangean decompositions and Dual Ascent solvers for Graph Matching 
%\title{Study of Lagrangean Decompositions and Dual Ascent Solvers for Graph Matching}
\title{A Study of Lagrangean Decompositions and Dual Ascent Solvers\\ for Graph Matching}

\author{Paul Swoboda, Carsten Rother, Hassan Abu Alhaija, Dagmar Kainm\"uller, Bogdan Savchynskyy}

\maketitle

\begin{abstract}
We study the quadratic assignment problem, in computer vision also known as graph matching.
Two leading solvers for this problem optimize the Lagrange decomposition duals with sub-gradient and dual ascent (also known as message passing) updates. We explore this direction further and propose several additional Lagrangean relaxations of the graph matching problem along with corresponding algorithms, which are all based on a common dual ascent framework. 
Our extensive empirical evaluation gives several theoretical insights and suggests a new state-of-the-art anytime solver for the considered problem. 
Our improvement over state-of-the-art is particularly visible on a new dataset with large-scale sparse problem instances containing more than $500$ graph nodes each.
\end{abstract}

\section{Introduction}
\label{sec:intro}
%{\color{red} Possibly add additional references about applications from~\cite{nguyen2015flexible}}
In computer vision and beyond, the quadratic assignment problem, known also as {\em graph matching}, {\em feature correspondence} and {\em feature matching}, has attracted great interest. 
This problem is similar to Maximum-A-Posteriori (MAP) inference on a discrete pairwise graphical model, also called conditional random field~(CRF) in the literature.
It differs in an additional uniqueness constraint: Each label can be taken at most once.
This uniqueness constraint makes it well-suited to attack e.g. tracking problems or shape matching.
In both cases feature points or object parts have to be matched between multiple frames one-to-one.
Unfortunately, the uniqueness constraint prevents naive application of efficient message passing solvers for MAP-inference to this problem.
For this reason, many dedicated graph matching solvers were developed, see related work below.

On the other hand, efficient dual block-coordinate ascent (also known as message passing) algorithms like TRW-S~\cite{TRWSKolmogorov} count among the most efficient solvers for MAP-inference in conditional random fields. Also, the graph matching problem, after possibly introducing many additional variables, can be stated as a MAP-inference problem in a standard pairwise CRF.
Such an approach already surpasses most state-of-the-art graph matching solvers.

Hence, it is desirable to devise specialized convergent message passing solvers exhibiting none of the drawbacks discussed above, i.e.\ 
(i)~directly operating on a compact representation of the graph matching problem.
(ii)~using techniques from the MAP-inference community to gain computational efficiency and

To achieve this goal, we propose (i)~several Lagrangean decompositions of the graph matching problem and (ii)~novel efficient message passing solvers for these relaxations.
We show their efficacy in an extensive empirical evaluation.

% {\color{red}
% To achieve this goal we apply a novel generalized convergent message passing framework~\cite{ConvergentMessagePassingNIPS}, in which we can state the graph matching problem directly and which enables us to devise efficient solvers for the obtained representations.
% }%TODO: this statement is unclear for reader without reading another paper in the parallel submission :(

\myparagraph{Related work}
The term {\em graph matching} refers to a number of different optimization problems in pattern recognition, see~\cite{ThirtyYearsOfGraphMatching} for a review. 
We mean the special version known unambiguously as {\em quadratic assignment problem} (QAP)~\cite{lawler1963quadratic}.
Recently, the graph matching was generalized to the {\em hypergraph matching} problem (see~\cite{nguyen2015flexible} and references therein), which match between more than two graphs.
%We also do not review here the related {\em hypergraph matching} problem (see~\cite{nguyen2015flexible} and references therein), which additionally to quadratic terms may contain terms of third, fourth and higher order. Although this problem is reducible to QAP, its consideration is beyond the scope of this work.

The quadratic assignment problem was first formulated in~\cite{QAPBeckman} back in $1957$. Since a number of NP-complete problems such as traveling salesman, maximal clique, graph isomorphism and graph partitioning can be straightforwardly reduced to QAP, this problem is NP-hard itself. Its importance for numerous applications boosted its analysis a lot: The (already aged) overview~\cite{AnalyticalSurveyQAP} contains $362$ references with over $150$ works suggesting new algorithms and over $100$ with new theoretical results related to this problem.

Nearly all possible solver paradigms were put to the test for QAP.
These include, but are not limited to, convex relaxations based on Lagrangean decompositions~\cite{karisch1995lower,GraphMatchingDDTorresaniEtAl}, linear~\cite{frieze1983quadratic,adams1994improved}, convex quadratic~\cite{anstreicher2001new} and semi-definite~\cite{zhao1998semidefinite,MRFSemidefiniteTorr,ProbabilisticSubgraphMatchingSchellewald} relaxations, which can be used either directly to obtain approximate solutions or just to provide lower bounds.
To tighten these bounds several cutting plane methods were proposed~\cite{bazaraa1979new,bazaraa1982use}.
On the other side, various primal heuristics, both (i) deterministic, such as local search~\cite{pardalos1993computational,TabuSearchGraphMatching}, graduated assignment methods~\cite{GraduatedAssignmentGold}, fixed point iterations~\cite{IntegerFixedPointGraphMatching}, spectral technique and its derivatives~\cite{SpectralTechniqueAssignmentLeordeanu,cho2014finding,UmeyamaGraphMatching,EigenDecompositionGraphMatchingZhao} and (ii) stochastic, like random walk~\cite{RandomWalksForGraphMatching} and Monte-Carlo sampling~\cite{lee2010graph,GraphMatchingSequentialMonteCarlo} were suggested to provide approximate solutions to the problem.
Altogether these methods serve as building blocks for exact branch-and-bound~\cite{gavett1966optimal,hahn1998branch,anstreicher2001solving} algorithms and other non-convex optimization methods~\cite{zaslavskiy2009path,FactorizedGraphMatching,GraduatedAssignmentGold}. The excellent surveys~\cite{AnalyticalSurveyQAP,BurkardQAP} contain further references.

As is usual for NP-hard problems, no single method can efficiently address all QAP instances. Different applications require different methods and we concentrate here on problem instances specific for computer vision. Traditionally within this community predominantly primal heuristics are used, since demand for low computational time usually dominates the need to obtain optimality guarantees. However, two recently proposed solvers~\cite{GraphMatchingDDTorresaniEtAl,HungarianBP} based on Lagrangean  decomposition (also known as {\em dual decomposition} in computer vision) have shown superior results and surpassed numerous state-of-the-art primal heuristics. 

{\em The dual decomposition solver}~\cite{GraphMatchingDDTorresaniEtAl} represents the problem as a combination of MAP-inference for binary CRFs, the linear assignment problem and a number of small-sized QAPs over few variables; Lagrangean multipliers connecting these subproblems are updated with the sub-gradient method. 
Although the solver demonstrates superior results on computer vision datasets, we suspect that its efficiency can be further improved by switching to a different update method, such as bundle~\cite{kiwiel1990proximity,kappes2012bundle} or block-coordinate ascent~\cite{wright2015coordinate}. 
This suspicion is based on comparison of such solvers for MAP-inference in CRFs~\cite{OpenGMBenchmark}
%, which is a related model to QAP, 
and similar observation related to other combinatorial optimization problems (see e.g.~\cite{ribeiro1986solving}).

{\em Hungarian Belief Propagation (HBP)}~\cite{HungarianBP} considers a combination of a multilabel CRF and a linear assignment as subproblems; Lagrange mutipliers are updated by a block-coordinate ascent (message passing) algorithm and the obtained lower bounds are employed inside a branch-and-bound solver. 
It is known~\cite{SRMPKolmogorov}, however, that efficiency of message passing significantly depends on the schedule of sending messages. 
Specifically, efficiency of dual ascent algorithms depends on selecting directions for the ascent (blocks of variables to optimize over) and the order in which these ascent operations are performed. 
Arguably, the underlying multilabel CRF subproblem is crucial and the message passing must deal with it efficiently.
However, HBP~\cite{HungarianBP} uses a message passing schedule similarly as in the MPLP algorithm~\cite{MPLP}, which was shown~\cite{SRMPKolmogorov,OpenGMBenchmark} to be significantly slower than the schedule of SRMP (TRW-S)~\cite{SRMPKolmogorov}.

\myparagraph{Contribution}
% existing decompositions: GM, AMCF - add cycling constraints
% new decomposition: AMP, inverse, coupling + cycling constraints
% Solvers - existing (SRMP) and new - fit into the general framework~\cite{}. 
% Comparing to DD we consider a different relaxation and use dual ascent (mesasge passing) technique
% Comparing to HBP we use another MP schedule + cycling constraints
% We provide an extensive experimental evaluation, we are the best
We study several Lagrangean decompositions of the graph matching problem. Some of these are known, e.g. the one used in the HBP algorithm~\cite{HungarianBP} and the one corresponding to the local polytope relaxation of the pairwise CRF representation of graph matching. The others have not been published so far, to our knowledge. 
For all these decompositions we provide efficient message passing (dual ascent) algorithms based on a recent message passing framework~\cite{ConvergentMessagePassingNIPS}. 
In the case of the local polytope relaxation our algorithm coincides with the SRMP method~\cite{SRMPKolmogorov}, a higher-order generalization of the famous TRW-S algorithm~\cite{TRWSKolmogorov}.  

Our experimental evaluation suggests a new state-of-the-art method for the graph matching problem, which outperforms both the dual decomposition~\cite{GraphMatchingDDTorresaniEtAl} and the HBP~\cite{HungarianBP} solvers. 
We propose tighter convex relaxations for all our methods.
%both to the hypergraph matching problem and 
%to tighter convex relaxations.
Also, we significantly improve performance of the HBP algorithm by changing its message passing schedule. 
%{\color{red} and by using a better rounding technique for obtaining primal integer-valued solution of the problem}.% TODO: da bin nicht sicher 

Proofs are given in the appendix. 
Code and datasets are availabe at \url{http://github.com/pawelswoboda/LP_MP}.

%{\color{red} 
\myparagraph{Notation.}
Undirected graphs are denoted by $G=(V,E)$, where $V$ is a finite {\em node set} and $E\subseteq{{V}\choose{2}}$ is {\em the edge set}.
The set of neighboring nodes of $v \in V$ w.r.t. graph $G$ is denoted by $\SN_G(v) := \{ u: uv \in E\}$.
The convex hull of a set $\SX \subset \R^n$ is denoted by $\conv(\SX)$.
%}

%We will use three fonts: bold font $\BG$ denotes notation related to factor graphs, calligraphic font $\SX$ integral sets and sans serif font $\FG$ MRFs.

%\input{graph_matching_related_work.tex}
%\section{Graph Matching} 
\section{CRFs and Graph Matching} 
\label{sec:GraphMatching}

% \begin{figure}
% \centering
% \includegraphics[width=\linewidth]{house_graph_matching.jpg}
% \caption{
% Graph matching on a house\cite{HouseDataset} instance. Each point on the left corresponds to a node $u \in \FG$ and is matched to exactly one point corresponding to label $s \in X_0$ on the right.
% {\color{red} The graph structure is not clear from the picture, it does not differ from those generated just with SIFT without any underlying graph.}
% }
% \label{fig:graph-matching}
% \end{figure}

First, we introduce conditional random fields and state the graph matching problem as one with additional uniqueness constraints.
Second, we consider an inverse formulation of the graph matching problem, which, after being coupled with the original formulation, often leads to faster algorithms.

\myparagraph{Conditional random fields (CRF).}
% graphical models: energy minimization

Let $\FG = (\FV,\FE)$ be an undirected graph. % with the finite sets $\FV$ of nodes and $\FE\subseteq{{\FV}\choose{2}}$ of edges.
With each node $u\in \FV$ we associate a variable~$x_u$ taking its values in a finite {\em set of labels} $\SX_u \subseteq \{(1,0,\ldots,0),(0,1,0,\ldots,0),\ldots,(0,\ldots,0,1)\}$. 
Hence, each label corresponds to a unit vector.
Notation $\SX_A$ denotes the Cartesian product $\prod_{u \in A\subseteq\FV} \SX_u$. 
A vector $x\in\SX_{\FV}$ with coordinates $(x_u)_{u\in\FV}$ is called a {\em labeling}. Likewise, we use the notation $x_A \in \SX_A$ (a special case being $x_{uv}\in\SX_{uv}\equiv\SX_u\times\SX_v$) to indicate part of a labeling.
Functions $\theta_u\colon \SX_u\to\R$, $u\in\FV$, and $\theta_{uv}\colon\SX_{uv}\to\R$, $uv\in\FE$, are {\em potentials}, which define a local quality of labels and label pairs. 
%The triple $(\SG, \SX_{\FV}, \theta)$ defines a {\em graphical model}, widely used in many areas of computer science~\cite{WainwrightBook}.

The {\em energy minimization} or {\em MAP-inference} problem for CRFs is
\begin{equation}
  \label{eq:energy-min}
  \min_{x \in \SX_{\FV}} \sum_{u \in \FV} \theta_u(x_u) + \sum_{uv \in \FE} \theta_{uv}(x_{uv})\,.
\end{equation}
The objective in~\eqref{eq:energy-min} is called {\em energy} of the CRF.

A great number of applied problems can be efficiently cast in the format~\eqref{eq:energy-min}, see e.g.~\cite{WainwrightBook,OpenGMBenchmark}.
This defines its importance for computer vision, machine learning and a number of other branches of science~\cite{WainwrightBook}.
While problem~\eqref{eq:energy-min} is NP-hard in general, many exact and approximate solvers were proposed~\cite{OpenGMBenchmark}.

\myparagraph{Graph Matching.}
Although the format of Problem~\eqref{eq:energy-min} allows us to express many practically important optimization tasks efficiently, some applications require the resulting labelings $x$ to satisfy additional constraints. 
In particular, for the graph matching problem no label may be taken twice.

Let a \emph{common universe} $\SL$ of labels be given such that $\SX_u \subseteq \SL$ $\forall u \in \FV$.
We require each label $s \in \SL$ to be taken at most once, i.e. $\abs{\{u \in \FV : x_u = s\}} \leq 1$.
In other words, we seek an injective mapping $(x_u)_{u \in \FV} :\FV \rightarrow \SL$.
This problem can be stated as
\begin{equation}
  \label{eq:GraphMatching}
  \min_{x \in \mathcal{X}_{\FV}} \sum_{u \in \FV} \theta_u(x_u) + \sum_{uv \in \FE} \theta_{uv}(x_{uv}) \quad \text{s.t. } x_u \neq x_v \ \forall u\neq v\,.
\end{equation}
%The constraint $x_u \neq x_v$ $\forall u \neq v$ is equivalent to requiring each label $s \in \SL$ to be taken at most once, i.e. $\abs{\{u \in \FV : x_u = s\}} \leq 1$.
%{\color{red}Fig.~\ref{fig:graph-matching} gives an example.}%TODO: add a fiure or emove me

Graph matching is NP-hard, since it is equivalent to MAP-inference for CRFs~\eqref{eq:energy-min} in the trivial case, when nodes of the graph contain mutually non-intersecting sets of labels.   

\myparagraph{Inverse Graph Matching}
%{\color{red}\cite{?}. I do not know any reference.}} 
A special case arises if the universe $\SL$ of labels to be matched has the same size as the set of nodes of the graph $|\SL| = |\FV|$.
%, as in Figure~\ref{fig:graph-matching}.
Then every injective mapping $(x_u)_{u \in \FV} : \FV \rightarrow \SL$ must also be a bijection. % due to the pigeon-hole principle.
Hence, every feasible labeling $x \in \SX_{\FV}$ corresponds to a permutation of $\FV$.
The graph matching problem~\eqref{eq:GraphMatching} can in this case also be approached in terms of the inverse permutation.
To this end let the {\em inverse graph} $\FG' = (\FV',\FE')$ be given by $\FV'=\SL$; the {\em inverse label set} $X'_{s} = \{v \in \FV : s \in X_v\}$ is associated with each node $s\in\FV'$; respectively $X'_{\SL} = \prod_{s \in \SL} X'_{s}$ is the set of {\em inverse labelings} and $X'_{st}$ denotes $X_s\times X_t$; 
the set of edges of the inverse graph is defined as $\FE' = \left\{ st\in \FV'\times \FV'\ :\ \exists x_{st} \in X'_{st} \text{ s.t. } x_{s} x_{t} \in \FE \right\}$.
The {\em inverse costs} $\theta'$ for $s,s'\in\FV'$, $x_s\in X'_s$ read:
\begin{equation*}
%\begin{split}
        \theta'_{s}(x_{s}) = \theta_{x_{s}}(s),
\ \
\theta'_{st}(x_{st}) = 
\left\{ \begin{array}{ll} 
        \theta_{x_{st}}(s,t), & x_{st} \in \FE \\
0, & \text{ otherwise.}
\end{array} \right.
%\end{split}
\end{equation*}
Consider the resulting \emph{inverse graph matching} problem
\begin{equation}
  \label{eq:InverseGraphMatching}
  \min_{x \in X'_{\SL}} \sum_{s \in \SL} \theta'_{s}(x_{s}) + \sum_{st \in \FE'} \theta'_{st}(x_{st}) \quad \text{s.t. } x_{s} \neq x_{t}\  \forall s\neq t\,.
%\sum_{i \in \FV} \theta_{x'_i}(i) + \sum_{ij \in \FE} \theta_{x'_{ij}}(i,j) \quad \text{s.t. } x'_i \neq x'_j\  \forall i\neq j
\end{equation}
Labeling $x \in \SX_{\FV}$ and inverse labeling $y \in X'_{\SL}$ correspond to each other iff 
$x_u = s \in \FV \Leftrightarrow y_{s} = u \in \FV$.
%{\color{red}An illustrative example is given in Figure~\ref{fig:OriginalAndInverseMatching}.}%TODO: either make it illustrative to remove me
%An illustrative example is given in Fig.~\ref{fig:OriginalAndInverseMatching}.

Note that when the edge set $\FE$ is sparse, the inverse edge set $\FE'$ may be not. 
In such a case, computational complexity of the inverse problem is higher than of the original one.

\section{Lagrangean Decompositions}
\label{sec:relaxation}
% energy minimization has the following LP relaxation:
% it can be represented in the following format, where X_i are ...
% Below, we will consider several LP relaxations of Graph Matching, which fall into the same format
% This has the advantage that we will be able to consider all the relaxations and corresponding algorithms in a formalized/generalized form

Since the graph matching problem~\eqref{eq:GraphMatching} is NP-hard, it is common to consider convex relaxations.
Below, we present three Lagrangean decompositon based relaxations of the problem. 
These can be applied to the original graph matching problem~\eqref{eq:GraphMatching}, to the inverse one~\eqref{eq:InverseGraphMatching} and to a combination of both.
Since all these relaxations are based on the famous local polytope relaxation~\cite{schlesinger1976syntactic,Werner07} of the MAP-inference for CRFs~\eqref{eq:energy-min}, we give a short overview of this relaxation first.

\myparagraph{Local Polytope for CRFs.} The MAP-inference problem~\eqref{eq:energy-min} can be represented as an integer linear program (ILP)~\cite{korte2012combinatorial} using an {\em overcomplete representation}~\cite{WainwrightBook} by grouping 
potentials corresponding to each node and edge into separate vectors. That is, $\theta_w(x_w)$, $w\in \FV\cup\FE$ stands for a vector with coordinates $(\theta_w(x_w))_{x_w\in X_w}$. The real-valued vectors $\mu_w$ have the same dimensionality as $\theta_w$ and stand for the "relaxed" version of~$x_w$. The corresponding linear programming (LP) relaxation reads:
\begin{align}\label{eq:LP:energy-min}
&\min\sum_{w\in\FV\cup\FE}\la\theta_w,\mu_w\ra\\
&\begin{array}{lll}
   \sum\limits_{x_w\in \SX_w} \mu_w(x_w) = 1,\ \mu_w(s)\ge 0,\  w\in \FV\cup\FE, s\in X_w\,,\\
   \sum\limits_{x_{v}\in \SX_{v}} \mu_{uv}(x_{uv}) = \mu_{u}(x_u),\    
uv\in\FE,\ u \in uv,\  x_u\in \SX_u\,.  \nonumber
\end{array} 
\end{align}
Constraints of~\eqref{eq:LP:energy-min} define the {\em local polytope} $\FL_{\FG}$.
Note that adding integrality constraints $\mu_w\in\{0,1\}^{|X_w|}$ makes the problem~\eqref{eq:LP:energy-min} equivalent to its combinatorial formulation~\eqref{eq:energy-min}.

\myparagraph{Integer Relaxed Pairwise Separable Linear
Programs (\IPSLP)} 
Below we describe a general problem format studied in~\cite{ConvergentMessagePassingNIPS}, which generalizes the local polytope relaxation~\eqref{eq:LP:energy-min}. Importantly, the same format fits also the Lagrangean decompositions of the graph matching problem, which we consider below. This makes it possible to consider all these relaxations at once from a general viewpoint.

Let a \emph{factor graph} $\SG = (\BF,\BE)$ consist of nodes $\BF = \{1,\ldots,k\}$, called \emph{factors} and edges $\BE$, called \emph{factor-edges}.
Let $\SX_i \subset \{0,1\}^{\dim(\SX_i)}$, $i \in \BF$, be sets of binary vectors and (ii) $A_{(i,j)} \subset \{0,1\}^{\dim(\SX_i) \times K_{ij}}$, $ij \in \BE$, $K_{ij} \in \N$ be matrices with binary entries, which map binary vectors from $\SX_i$ into binary vectors from $\{0,1\}^{K_{ij}}$, i.e., $A_{(i,j)} \colon \SX_i \rightarrow \{0,1\}^{K_{ij}}$.
The \IPSLP\ is a class of problems, which factorize according to $\BG$.
\begin{align}
\label{eq:FactorGraphPrimal}
        \min_{\mu \in \Lambda_{\BG}} & \sum_{i \in \BF} \la \theta_i,\mu_i \ra \\
        \Lambda_{\BG} := &
\left\{
\begin{pmatrix} \mu_1 & \ldots & \mu_k \end{pmatrix}  : 
\begin{array}{ll}
\mu_i \in \conv(\SX_i) & i \in \BF \\
A_{(i,j)}\mu_i = A_{(j,i)} \mu_j & \forall ij \in \BE %\subseteq \begin{pmatrix} \BF \\ 2 \end{pmatrix}
\end{array}
\vspace{3pt}\right\}. \nonumber
\end{align}
Constraints $A_{(i,j)}\mu_i = A_{(j,i)} \mu_j$ are associated with each factor-edge and are called {\em coupling constraints}.
When representing the local polytope relaxation~\eqref{eq:LP:energy-min} as~\eqref{eq:FactorGraphPrimal} we assume $\BF=\FV\cup\FE$ and $\BE=\{\{u, uv\}, \{v, uv\} : uv \in \FE\}$. The convex hull of $X_w$ is fully defined by the first line of constraints in~\eqref{eq:LP:energy-min}, since $X_w$ constitutes a set of unit binary vectors. The second line of constraints in~\eqref{eq:LP:energy-min} defines the coupling constraints.

We use variable names $\mu$ for (in general) non-binary vectors $\mu_i \in \conv(\SX_i)$ and $x$ for binary ones $x_i \in \SX_i$, $i\in\BF$.

% {\color{red}
% \begin{remark}
% The decomposition~\eqref{eq:FactorGraphPrimal} is a special case of dual decomposition.
% Factors $i \in \BF$ correspond to subproblems and coupling constraints $ij \in \BE$ to Lagrangian variables.
% We additionally require that the subproblems are binary and the Lagrangian variables between the subproblems must correspond to linear $\{0,1\}$-constraints mapping $\{0,1\}$-vectors to $\{0,1\}$-ones again.
% The reason we use the class of \IPSLP\ to model the graph matching problem is that it is possible to do fast inference in them~\cite{ConvergentMessagePassingNIPS}.
% \end{remark}
% 
% A natural relaxation of the graph matching problem~\eqref{eq:GraphMatching} is
% \begin{equation}\label{eq:graph-matching-relaxation}
%         \{      \mu \in \FL_{\FG} : 
%         \sum_{u \in \FV: s \in X_u} \mu_{u}(s) \leq 1 \quad \forall s \in \SL \}\,.
% \end{equation}
% The additional constraints directly take care of the uniqueness constraint by enforcing each label to be taken at most once.
% Unfortunately, the above relaxation does not fit the \IPSLP format~\eqref{eq:FactorGraphPrimal}. 
% Below we will investigate several ways to transform~\eqref{eq:graph-matching-relaxation} into an \IPSLP\ and optimize it with message passing.
% }%TODO: I am not sure whether we need it here.

\subsection{Graph Matching Problem Relaxations.}\label{sec:graph-matching-relaxations}
Below, we describe three relaxations of the graph matching~\eqref{eq:GraphMatching} problem, which fit the \IPSLP~\eqref{eq:FactorGraphPrimal} format. The first one results in a standard local polytope relaxation~\eqref{eq:LP:energy-min} of a specially constructed CRF, the second one utilizes additional coupling constraints on top of~\eqref{eq:LP:energy-min}, while the third approach uses a network flow subproblem. Additionally, we use the inverse formulation~\eqref{eq:InverseGraphMatching} and build two additional \IPSLP s.

% We propose three relaxations of the graph matching~\eqref{eq:GraphMatching} problem in terms of an \IPSLP~\eqref{eq:FactorGraphPrimal} are proposed.
% The first results in a standard local polytope relaxation~\eqref{eq:ILP:linear-constraints} of a discrete pairwise MRF,
% the second one utilizes additional coupling constraints on top of~\eqref{eq:ILP:linear-constraints}, while the third approach uses a network flow subproblem.
% Also, we can combine the inverse formulation~\eqref{eq:InverseGraphMatching} and build an extended \IPSLP.

%TODO: who did it before?
\myparagraph{(R1) Graph Matching as CRF.}
To build a CRF equivalent to the graph matching we start with the underlying CRF as in~\eqref{eq:GraphMatching}  and express the uniqueness constraints in the edge factors. To this end we
(i)~extend the edge set $\FE$ with new edges connecting any two nodes having at least one common label, i.e.\ $\hat\FE: = \FE\cup\{uv\in{{\FV}\choose{2}} \colon \SX_u \cap \SX_v \neq \emptyset\}$; (ii)~assign edge potentials $\theta_{uv} \equiv 0$ to all new edges $\hat\FE\backslash\FE$; (iii)~for all $uv\in \hat\FE$ we assign $\theta_{uv}(x,x): = \infty$ $\forall x \in \SX_u \cap \SX_v$.
Any solution of the resulting CRF~\eqref{eq:energy-min} with cost $< \infty$ is an assignment.
The relaxation in terms of an \IPSLP\ is the local polytope~\eqref{eq:LP:energy-min}.

This approach results in general in a quadratic number of additional edge potentials, which may become intractable as the size of the graph matching problem grows.
%{\color{red}The work~\cite{CoveringTreesLowerBoundQuadraticAssignmentJarkony} has followed a similar approach, except that no additional pairwise potentials were added, but a primal assignment was obtained in a post-processing step.}

\myparagraph{(R2) Relaxation with Label Factors.}
For each label ${s \in \SL}$ we introduce an additional \emph{label factor}, which keeps track of nodes which assign label $s$.
The label set of this factor $X_s := \{u \in \FV : s \in \SX_u\} \cup \{\#\}$ consists of those nodes $u \in \FV$ which can be assigned label $s$ and an additional dummy node $\#$ representing non-assignment of label $s$. Label \# is necessary, as not every label needs to be taken.
The set of factors becomes $\BF = \FV \cup \FE \cup \SL$,
with the coupling constraint set $\BE = \{\{u,uv\}  ,  \{v, uv\} : uv \in \FE\} \cup \{\{u,l\} : u \in \FV, l \in \SX_u \}$.
The resulting \IPSLP\ formulation reads
\begin{align} \label{eq:LabelFactorsIPSLP}
\tag{R2}
& \min\sum_{w\in\FV\cup\FE}\la\theta_w,\mu_w\ra + \sum_{s\in\SL}\la\tilde\theta_s,\tilde\mu_s\ra\\
& \begin{array}{ll}
   \mu \in \FL_{\FG} \\
   \tilde\mu_{s} \in \conv(X_s), & s \in \SL\\
   \mu_{u}(s) = \tilde\mu_{s}(u), & s \in X_u\,.
  \end{array} \nonumber
\end{align}
Here we introduced additional potentials $\tilde\theta_s$ for the label factor. 
Initially, we set $\tilde\theta_s\equiv 0$.
%, however during the optimization, potentials $\theta$ are redistributed between $\theta$ and $\tilde\theta$ as {\color{red} is usual for dual decomposition algorithms}.

\myparagraph{(R3) Relaxation with a Network Flow Factor.}
If one ignores the edge potentials $\theta_{uv}$ in~\eqref{eq:GraphMatching}, the problem can be equivalently reformulated as bipartite matching~\cite{AhujaMagnantiOrlinNetworkFlows}:
\begin{align}
        \label{eq:matching-constraints}
        &\min_{\mu\in\SM}\sum_{u \in \FV}\la\theta_u,\mu_u\ra,\quad \text{where}\\
        &\SM = \left\{ (\mu_u)_{u \in \FV} \geq 0 \colon 
        \begin{array}{ll}
                \sum_{s \in X_u} \mu_u(s) = 1, u \in \FV\\
                \sum_{u \in \FV, s \in X_u} \mu_u(s) \leq 1, s \in \SL\\
        \end{array}
        \right\}  \nonumber
\end{align}
Here we substituted the uniqueness constraints with the linear inequalities $\sum_{u \in \FV, s \in X_u} \mu_u(s) \leq 1$, which is equivalent for $\mu_u\in\{0,1\}^{|X_u|}$. 
%It is known~\cite{AhujaMagnantiOrlinNetworkFlows} that all vertices of the polytope $\SM$ correspond to binary vectors and therefore this transformation does not lead to a relaxation for itself. 
It is known that $\SM$ is the convex hull of all binary vectors satisfying the conditions of $\SM$~\cite{AhujaMagnantiOrlinNetworkFlows}, i.e. $\conv(\SM \cap \{0,1\}^{\dim(\SM)}) = \SM$.
%~\footnote{This is due to the well-known fact that each polytope is equal to the convex hull of its vertices~\cite{korte2012combinatorial}}. 
Therefore $\SM$ fits into the \IPSLP\ framework.
Crucially for an efficient implementation,~\eqref{eq:matching-constraints} can be efficiently solved by minimum cost flow solvers~\cite{AhujaMagnantiOrlinNetworkFlows}.

Below we treat~\eqref{eq:matching-constraints} as a separate factor $\SM$ and link it with~\eqref{eq:LP:energy-min} to obtain an \IPSLP.
Its factor graph is defined by $\BF = \FV \cup \FE \cup \{ \SM \}$ and $\BE = \{\{u,uv\}  ,  \{v, uv\} : uv \in \FE\} \cup \{\{u,\SM\} : u \in \FV \}$.
The resulting \IPSLP\ formulation is
\begin{align} \label{eq:local-polytope-matching-factor}
\tag{R3}
&        \min\sum_{w\in\FV\cup\FE}\la\theta_w,\mu_w\ra + \sum_{u\in\FV}\la\tilde\theta_u,\tilde\mu_u\ra\\
&        \mu \in \FL_{\FG},\ \ \tilde\mu \in \SM, \nonumber \\
&        \tilde\mu_u(s) = \mu_u(s),\ u\in\FV,\ s\in X_u \nonumber\,.
\end{align}
Initially, we set $\tilde\theta \equiv 0$.

Representation~\eqref{eq:matching-constraints} for the uniqueness constraints has been already used e.g., in~\cite{CombinatorialOptimizationMaxProductDuchi}.
However their optimization technique lacks both convergence guarantees and monotonicity of a lower bound, which our methods possess. 
The work~\cite{HungarianBP} considered the Lagrange dual of~\eqref{eq:local-polytope-matching-factor} as a relaxation the graph matching problem. Their relaxation is equivalent to~\eqref{eq:local-polytope-matching-factor}, but their algorithm differs from ours. We refer to Section~\ref{sec:graph-matching-algs} for a discussion of the differences and to Section~\ref{sec:Experiments} for an experimental comparison.

%Note that relaxation~\eqref{eq:local-polytope-matching-factor} cannot be stated as a graphical model without incurring exponential blowup due to the global factor~\eqref{eq:matching-constraints}.
% In the work~\cite{CombinatorialOptimizationMaxProductDuchi} factor~\eqref{eq:matching-constraints} was held implicitly, avoiding blowup, but the update operations did not result in a monotonic series of dual lower bounds.
% In~\cite{HungarianBP} the authors used the underlying bipartite matching problem to optimize over the dual of~\eqref{eq:matching-constraints} with message passing. They {\color{red} derived their message updates in an ad-hoc way} %TODO: sounds agressive
% to monotonically improve the lower bound.

\myparagraph{(R4-R5) Coupling Original Graph Matching~\eqref{eq:GraphMatching} and its Inverse~\eqref{eq:InverseGraphMatching}.}
In the special case when $|\SL| = |\FV|$ we may solve the inverse graph matching problem~\eqref{eq:InverseGraphMatching} instead of the original one~\eqref{eq:GraphMatching}.
Another alternative is to solve both problems simultaneously and couple them together by requiring that the labeling of~\eqref{eq:GraphMatching} is the inverse permutation for the labeling from~\eqref{eq:InverseGraphMatching}.
Such an approach doubles the problem size, yet it may result in a smaller number of iterations required to obtain convergence. %on the other hand solving might actually become faster, as more information is propagated through the resulting coupling constraints.
This approach works both for relaxations~\eqref{eq:LabelFactorsIPSLP} and~\eqref{eq:local-polytope-matching-factor}.

The resulting coupled \IPSLP\ for~\eqref{eq:LabelFactorsIPSLP} reads
\begin{align}
\tag{R4}
 & \min_{\mu,\mu'}\sum_{w\in\FV\cup\FE}\la\theta_w,\mu_w\ra + \hspace{-5pt}\sum_{w\in\FV'\cup\FE'}\la\theta'_w,\mu'_w\ra\\
 & \mu\in\FL_{\FG},\ \mu'\in\FL_{\FG'}\nonumber \\
 & \forall u \in \FV, u' \in \SX_u\colon \mu_u(u') = \mu'_{u'}(u)\,.\nonumber
\end{align}

Here the role of label factors in~\eqref{eq:LabelFactorsIPSLP} has been taken over by the node factors of the inverse graph matching~\eqref{eq:InverseGraphMatching}.  
We distribute the costs equally among $\theta$ and $\theta'$ initially.
%The inverse potentials $\theta'_w$ are equal to zero at the beginning. During optimization the values of potentials $\theta_w$ are redistributed between $\theta_w$ and $\theta'_w$.

Another coupled \IPSLP, corresponding to~\eqref{eq:local-polytope-matching-factor} reads
\begin{align}
 & \min_{\mu,\mu',\tilde \mu}\sum_{w\in\FV\cup\FE}\la\theta_w,\mu_w\ra +\hspace{-5pt} \sum_{w\in\FV'\cup\FE'}\la\theta'_w,\mu'_w\ra + \sum_{u\in\FV}\la\tilde\theta_u,\tilde\mu_u\ra \nonumber \\
 & \mu\in\FL_{\FG},\ \mu'\in\FL_{\FG'}, \tilde\mu\in\SM \tag{R5} \\
 & \forall u \in \FV, u' \in \SX_u\colon \mu_u(u') = \tilde\mu_{u}(u'),\ \mu'_{u'}(u) = \tilde\mu_{u}(u') \nonumber 
\end{align}
Here the network flow factor $\SM$ controls consistency of the original $\mu$ and inverse labelings $\mu'$. Initially, we set $\tilde\theta\equiv0$ and distribute costs in $\theta$ and $\theta'$ equally.
%andBoth $\theta'$ and $\tilde\theta$ are zero at the beginning and may get non-zero weights in the course of optimization.

The optimal values obtained by relaxations (R1) -- (R5) may deliver differing bounds to~\eqref{eq:GraphMatching}, as characterized below.
\begin{proposition}\label{prop:RelaxationRelation}
        (R2) = (R3) and (R4) = (R5).
        Relaxation (R1) is weaker than (R2) and (R3).
        %(R2) is equal to (R3), as is (R4) and (R5).
%\begin{equation}
%\text{(R2)} =  \text{(R3)} \geq \text{(R1)}, \text{(R4)} = \text{(R5)} \,.
%\end{equation}
\end{proposition}
%\textcolor{red}{ Certainly, relaxations (R2) and (R3) are equivalent, as well as (R4) and (R5). But I do not think the inverse formulation is equivalent -> Indeed, when (R2) is e.g. a tree, we could solve it exactly with dynamic programming, but the inverse will possibly not be a tree. The joint relaxation hence may be tighter than the original one.%$} %TODO: supplement, or no time for that? It is easy - uniqueness constraint it is a higher order factor, which is more tight, than its representation with 2nd order factors

\section{General Algorithm} 
\label{sec:Algorithms}

In this section we define a general algorithm for \IPSLP\ problems~\eqref{eq:FactorGraphPrimal}, which is applicable to the decompositions (R1)--(R5) of the graph matching problem considered in Section~\ref{sec:graph-matching-relaxations}.  Our algorithm is a simplified version of the algorithm~\cite{ConvergentMessagePassingNIPS}, where we fixed several parameters to the values common to the relaxations (R1)--(R5).

Instead of optimizing \IPSLP~\eqref{eq:FactorGraphPrimal} directly, we consider its Lagrangean dual w.r.t. the coupling constraints $A_{(i,j)}\mu_i = A_{(j,i)} \mu_j$.
The \IPSLP\ problem~\eqref{eq:FactorGraphPrimal} can be shortly written as $\min_{\mu}\{\la \theta,\mu\ra,\ \text{s.t.}\ A\mu=0, \mu\in P\}$, where $\mu$ stands for $(\mu_i)_{i=1}^k$,  $A\mu=0$ represents all coupling constraints $A_{(i,j)}\mu_i - A_{(j,i)} \mu_j =0$ and $P$ denotes a polytope encapsulating the rest of constraints. By dualizing $A\mu=0$ with a vector of Lagrange multipliers $\Delta$ one obtains the Lagrange function $\la \theta,\mu\ra-\la\Delta,A\mu\ra=\la \theta-A^{\top}\Delta,\mu\ra$. After introducing $\theta^{\Delta}:=\theta-A^{\top}\Delta$ the dual objective reads $D(\Delta)=\min_{\mu}\{\la\theta^{\Delta},\mu\ra,\ \text{s.t.}\ \mu\in P\}$. It is well-known~\cite{boyd2004convex} that $D(\Delta)\le \la\theta,\mu\ra$ for any feasible $\mu$ and the dual problem consists in maximizing $D(\Delta)$ over~$\Delta$.
Going from $\theta$ to $\theta^{\Delta}$ is called an {\em equivalent transformation} or {\em reparametrization} in the literature. In the CRF-literature it is also known as {\em message passing}.

%Note that $A\mu=0$ implies $\la\theta^{\Delta},\mu\ra = \la\theta,\mu\ra$ for any~$\Delta$. In particular, $A\mu=0$ holds for integer-valued $\mu$ corresponding to solutions of the combinatorial, non relaxed problem~\eqref{eq:GraphMatching}. This means that potentials $\theta$ and $\theta^{\Delta}$ are {\em equivalent}, i.e.\ switching between them does not change the cost of any labeling, independent of $\Delta$. In contrast, the dual objective $D(\Delta)$ depends on $\Delta$. Therefore, optimization of the dual $D(\Delta)$ one can see as {\em a reparametrization} or {\em an equivalent transformation} of the initial problem. Another popular name for this process is {\em message passing}, which is widely used, e.g., for MAP-inference in CRFs. 

%Below, we define an algorithm, which performs such reparametrizations of the problem in a way to guarantee a monotonous non-decreasing of the dual objective $D$.

Now we apply the above considerations to the general \IPSLP\ problem~\eqref{eq:FactorGraphPrimal}.
%Let $i,j\in\BF$ be two neighboring factors in the factor-graph $\BG$. 
Specifically, let $i,j\in\BF$ be two neighboring factors in the factor-graph $\BG$. 
Then for any $\mu_i$ and $\mu_j$ satisfying the coupling constraint for edge $ij \in \BE$ 
\begin{equation*}
        \begin{array}{rl}
                & \la \theta_i, \mu_i \ra + \la \theta_j, \mu_j \ra \\
                = & \la \theta_i, \mu_i \ra + \la \theta_j, \mu_j \ra + \underbrace{   \la \Delta_{(i,j)}, A_{(i,j)} \mu_i - A_{(j,i)} \mu_j \ra   }_{=0} \\
                = & \la \theta_i + A_{(i,j)}^\T \Delta_{(i,j)}, \mu_i \ra + \la \theta_j - A_{(i,j)}^\T \Delta_{(i,j)}, \mu_j \ra \,.
        \end{array}
\end{equation*}
%Therefore, the corresponding equivalent transformation reads 
%\begin{equation}
% \theta_i\to \theta_i + A_{(i,j)}^\T \Delta_{(i,j)}\ \text{and}\ \ \theta_j\to \theta_j - A_{(j,i)}^\T \Delta_{(i,j)}\,.
%\end{equation}
The values and sign of the Lagrange multiplies $\Delta_{(i,j)}$ define how much cost is "sent" from $j$ to $i$ or the other way around.
%Similarly, 
When we consider a subset $J\subseteq\SN_{\BG}(i)$ of the neighboring factors for $i$, the resulting equivalent transformation reads:
\begin{equation}\label{eq:equiv-transform-i-J}
 \hspace{-5pt}\theta_i\to \theta_i + \sum_{j\in J}A_{(i,j)}^\T \Delta_{(i,j)}\ \text{and}\ \theta_j\to \theta_j - A_{(j,i)}^\T \Delta_{(i,j)}.
\end{equation}

We are interested in $\Delta_{(i,j)}$ which improve the dual.
Below we define a subclass of such messages for the same setting as in~\eqref{eq:equiv-transform-i-J}:
\begin{definition}\label{def:admissible-message} Messages $\Delta(i,j)$, $j\in J$, are called {\em admissible}, if there exists $x_i^* \in \argmin\limits_{x_i\in\SX_i} \la \theta_i, \mu_i \ra \cap \argmin\limits_{x_i\in\SX_i} \la \theta^{\Delta}_i, \mu_i \ra$ and additionally
\begin{equation}\label{equ:allowed-delta}
 \Delta_{(i,j)}(s) 
\begin{cases}
\ge 0, & \nu(s) =1\\
\le 0, & \nu(s) =0
\end{cases},
\text{where}\ \nu:=A_{(i,j)} x^*_i\,.
\end{equation} 
\end{definition}
We denote the set of admissible vectors by $AD(\theta_i,x^*_i,J)$.
\begin{lemma}[\cite{ConvergentMessagePassingNIPS}]\label{lem:admissible-dual-monotone}
 Admissible messages do not decrease the dual value, i.e., $\Delta\in AD(\theta^{\phi}_i,x^*_i,J)$ implies $D(0)\le D(\Delta)$.
\end{lemma}

\begin{example}\label{example:addmissible-dual} Let us apply Definition~\ref{def:admissible-message} to the local polytope relaxation~\eqref{eq:LP:energy-min} of CRFs.
 Let $ij$ correspond to $\{u,uv\}$, where $u\in\FV$ is some node and $uv\in\FE$ is any of its incident edges and $J=\{j\}$.
 Then $x_i^*$ corresponds to a locally optimal label $x^*_u\in\arg\min_{s\in X_u}\theta_u(s)$ and $\nu(s)=\llbracket s= x^*_u\rrbracket$. Therefore we may assign $\Delta_{u,uv}(s)$ to any value from $[0,\theta_u(x^*_u) - \theta_u(s)]$. This assures that \eqref{equ:allowed-delta} is fulfilled and $x^*_u$ remains a locally optimal label after reparametrization even if there are multiple optima in $X_u$.
\end{example}
%{\color{red}To pass messages, we employ the following algorithmic routine. It expects a factor $i\in \BF$ from which to send messages to a set of neighbors $J \subseteq \SN_{\BG}(i)$.}
\myparagraph{Sending Messages.} Procedure~\ref{alg:factor-optimization} represents an elementary step of our optimization algorithm. It consists of sending messages from a node $i$ to a subset of its neighbors $J$.
\begin{algorithm}
\SetAlgorithmName{Procedure}{}\\
        \caption{Send messages from $i \in \BF$ to $J \subseteq \SN_{\BG}(i)$}
        \label{alg:factor-optimization}        
\textbf{Optimize factor:} \label{alg1:optimize-factor}
$
x_i^* \in \argmin\limits_{x_i\in\SX_i} \la \theta_i, \mu_i \ra 
$\\
% \begin{equation}\label{eq:FactorOptimization} %\label{alg1:optimize-factor}
% \text{\bf Optimize factor:} 
% x_i^* \in \argmin\limits_{x_i\in\SX_i} \la \theta_i, \mu_i \ra 
% \end{equation}\\
%\begin{equation}\label{equ:delta-def}
        \textbf{Choose} $\delta \in \R^{d_i}\ \text{s.t.}\ \delta(s) \left\{ \begin{array}{ll} \geq 0, & x^*_i(s) = 1 \\ \leq 0,& x^*_i(s) = 0 \end{array} \right.$ \label{alg1:delta}\\
%\end{equation}
\textbf{Maximize admissible messages to $J$:}
        \label{alg1:optimize-messages}
\begin{equation}\label{eq:ReverseFactorOptimization}
\Delta_{(i,J)}:=(\Delta_{(i,j)})_{j\in J} \in 
\argmax\limits_{\hat\Delta \in D(\theta^{\phi}_i,x^*_i,J)} \la \delta,\theta_i^{\hat\Delta} \ra 
\end{equation}\\
\textbf{Update $\theta^i$ and $\theta_j$, $j\in J$, according to~\eqref{eq:equiv-transform-i-J}} \label{alg1:update-costs}
\end{algorithm}

Procedure~\ref{alg:factor-optimization} first computes an optimal labeling for the factor $i$ in line~\ref{alg1:optimize-factor}, then computes message updates in~\eqref{eq:ReverseFactorOptimization} and finally updates the costs $\theta$ in line~\ref{alg1:update-costs}.
The costs $\delta$ in line~\ref{alg1:delta} are chosen as $\pm 1$, except when $i = \SM$ is the network flow factor for (R3) and (R5). 
In this case, we choose $\delta(u,x_u) = \left\{ \begin{array}{ll} 0,& x_u = x_u^* \\ 1-\abs{X_u},& x_u \neq x_u^*\end{array} \right.$.
%{\color{red} Comment on $\delta$!}

Computation~\eqref{eq:ReverseFactorOptimization} provides a maximally possible admissible message from $i$ to $\{J\}$. Essentially, it makes the cost vector of the factor~$i$ as uniform as possible. 
So, in the setting of Example~\ref{example:addmissible-dual} $\Delta_{u,uv}(s)$  becomes equal to $\theta_u(x^*_u) - \theta_u(s)$ and therefore $\theta_i^{\Delta}(s)=\theta_u(x^*_u)$ for all $s\in X_u$.
%In the appendix we give explicit forms and, if exist, closed-form solutions of~\eqref{eq:ReverseFactorOptimization} for all combinations of neighboring factors appearing in the considered in Section~\ref{sec:graph-matching-relaxations} relaxations~(R1)-(R5).
%\textcolor{red}{
%Table~\ref{table:ReparametrizationAdjustment} gives explicit forms and, if exist, closed-form solutions of~\eqref{eq:ReverseFactorOptimization} for all combinations of neighboring factors appearing in the considered in Section~\ref{sec:graph-matching-relaxations} relaxations~(R1)-(R5).
%}
Since the result of~\eqref{eq:ReverseFactorOptimization} is an admissible message, Procedure~\ref{alg:factor-optimization} never decreases the dual objective, as follows from Lemma~\ref{lem:admissible-dual-monotone}.

\myparagraph{Dual Ascent Algorithm.} 
Let the notation $\{j_1,\dots,j_n\}_<$ stand for an ordered set such that $j_k < j_{k+1}$, $k=1,\dots,n$.
%Procedure~\ref{alg:factor-optimization} is a key part of Algorithm~\ref{alg:message-passing}, which aims to optimize the dual problem. 
Algorithm~\ref{alg:message-passing} below goes over some of the factors $i \in \BF$ in a pre-specified order and calls Procedure~\ref{alg:factor-optimization} to send or receive messages to/from some of the neighbors. 
% Essentially, 
% we call Procedure~\ref{alg:factor-optimization} with varying input $i \in \BF$ and $J \subseteq \SN_{\BG}(i)$ as follows:
% 
% We sequentially call Algorithm~\ref{alg:factor-optimization} with varying input $i \in \BF$ and $J \subseteq \SN_{\BG}(i)$ as follows:
% Given an ordered sequence of factors, we iterate over those factors in forward and backward order and invoke Algorithm~\ref{alg:factor-optimization} suitably.
\begin{algorithm}
        \label{alg:message-passing}
        \caption{Dual Ascent for \IPSLP}
        {\bf{Input:}} ${I = \{i_1,\ldots,i_k\}_< \subseteq \BF}$\,, $(J_r(i) \subseteq \SN_{\BG}(i))_{i \in I}$\,, $(J_s(i) \subseteq \SN_{\BG}(i))_{i \in I}$\\
        \For{$iter=1,\ldots$}{
        \For{$i = i_1,\ldots,i_k$} {
                \textbf{Receive messages:} \\
        \For{$j \in J_r(i)$} {
                Call Algorithm~\ref{alg:factor-optimization} with input $(j,\{i\})$.
        }
                \textbf{Send messages:} \\
                Call Algorithm~\ref{alg:factor-optimization} with input $(i,J_s(i))$.
        }
        \text{Reverse~the~order~of}~$i_1,\ldots,i_k$~\text{and~exchange}~${J_r \leftrightarrow J_s}$% and receive and send messages as above.
}
\end{algorithm}

Algorithm~\ref{alg:message-passing} works as follows: We choose an ordered subset of factors $\{i_1,\ldots,i_k\}_{<}$.
For each factor $i \in \BF$ we select a neighborhood $J_r(i) \subseteq \SN_{\BG}(i)$ of factors from which to receive messages and a neighborhood $J_s(i)$ to which messages are sent by Procedure~\ref{alg:factor-optimization}.
We run Algorithm~\ref{alg:message-passing} on $\{i_1,\ldots,i_k\}_{<}$ (forward direction) and $\{i_k,\ldots,i_1\}_{<}$ (backward direction) alternatingly until some stopping condition is met.
Since Algorithm~\ref{alg:message-passing} reparametrizes the problem by Procedure~\ref{alg:factor-optimization} only and the latter is guaranteed to not decrease the dual, so is Algorithm~\ref{alg:message-passing}. We refer to~\cite{ConvergentMessagePassingNIPS} for further theoretical properties of Algorithm~\ref{alg:message-passing}.

%\myparagraph{Graph matching algorithms.}
\section{Graph matching algorithms.}\label{sec:graph-matching-algs}

For each of the relaxations (R1)-(R5) of the graph matching problem we detail parameters of Algorithm~\ref{alg:message-passing} used in our experiments:
we define the sets $I$, $J_r(i)$, $J_s(i)$.
%We give  of the problem~\eqref{eq:ReverseFactorOptimization} or its closed-form solution, if the latter exists, for each possible combination of neighboring factors.

\myparagraph{Algorithm Names.} We use the following shortcuts for specializations of Algorithm~\ref{alg:message-passing} to the relaxations (R1)-(R5):
\textbf{GM} corresponds to (R1), \textbf{AMP} to (R2), \textbf{AMCF} to (R3). To obtain the relaxations (R1-R3) we use either the original graph, as in~\eqref{eq:GraphMatching}, or an inverse one, as in~\eqref{eq:InverseGraphMatching}. These options are denoted by suffixes \textbf{-O} and \textbf{-I} respectively. Additionally, the two coupled relaxations (R4) and (R5), are addressed by algorithms \textbf{AMP-C} and \textbf{AMCF-C} respectively. All in all, we have eight algorithms \textbf{GM-O}, \textbf{GM-I}, \textbf{AMP-O}, \textbf{AMP-I}, \textbf{AMP-C}, \textbf{AMCF-O}, \textbf{AMCF-I} and \textbf{AMCF-C}.
%{\color{red} What about AMCF-C instead of AMCF-C? }

\myparagraph{The sets $I$, $J_r(i)$ and $J_s(i)$}are defined in Table~\ref{table:factor-order}. For algorithms with the suffix {\bf -I} the values are the same as for those with {\bf -O}, but corresponding to the inverse graph.

We assume the order of graph nodes $\FV := \{u_1,\ldots,u_n\}_<$ and labels $\SL := \{s_1,\ldots,s_{\abs{\SL}}\}_{<}$ to be given a priori. We define $u_n < \SM < s_1$ for the matching factor $\SM$ and  $u < uv < v$ for the edge factors $uv\in\FE$. Similarly, we define $s < ss' < s'$ for all edge factors $ss' \in \FE'$ in the inverse graph. 
We extend the resulting partial order to a total one, e.g., by topological sort. For $i\in\BF$ we define $\SN_{\BG}(i)_< := \{j \in \SN_{\BG}(i) : j < i\}$ and $\SN_{\BG}(i)_> := \SN_{\BG}(i) \backslash J_r(i)$ as the sets of preceding and subsequent factors.

Sending a message by some factor automatically implies receiving this message by another, coupled factor. Therefore, there is no need to go over all factors in Algorithm~\ref{alg:message-passing}. In particular, edge-factors are coupled to node-factors only, therefore processing all node factors in Algorithm~\ref{alg:message-passing} automatically means updating all edge-factors as well. In the processing order and selection of the sets $J_r(i)$ and $J_s(s)$ we follow the most efficient MAP-solvers TRWS~\cite{TRWSKolmogorov} and SRMP~\cite{SRMPKolmogorov} (the latter is a generalization of TRW-S to higher order models and has a slightly different implementation for pairwise CRFs~\eqref{eq:energy-min}). 
In the special case when all nodes contain disjoint subsets of labels the graph matching problem~\eqref{eq:GraphMatching} turns into MAP-inference in CRFs~\eqref{eq:energy-min}.
Then all our algorithms {\bf GM}, {\bf AMP} and {\bf AMCF} reduce to SRMP~\cite{SRMPKolmogorov}.

It is worth mentioning that for CRFs there exist algorithms, such as MPLP~\cite{MPLP}, which go over edge-factors only and in this way implicitly process also node-factors. As empirically shown in SRMP~\cite{SRMPKolmogorov}, MPLP is usually slower than SRMP. In Section~\ref{sec:Experiments} we show that our methods also favorably compare to the recently proposed HBP~\cite{HungarianBP}, which is similar to {\bf AMCF-O}, but uses an MPLP-like processing schedule.  

% {\color{red} It is sufficient only to guarantee that all factors receive and send messages. }
% {\color{red} Why we do not go over pairwise factors, why and how $M$  is processed only once per iteration}
% 
% \begin{remark}
% Algorithms~\textbf{GM},~\textbf{AMP} and~\textbf{AMCF} do not visit pairwise factors, similarly to TRWS~\cite{TRWSKolmogorov} and SRMP~\cite{SRMPKolmogorov}.
% The algorithm MPLP~\cite{MPLP} on the other hand, does visit pairwise factors but omits the unary ones, similary to the hungarian belief propagation~\cite{HungarianBP} for solving~\eqref{eq:GraphMatching}.
% It has been shown in~\cite{SRMPKolmogorov} that visiting pairwise factors instead of unary ones in message passing leads to slower convergence by an order of magnitude for most problems.
% We will confirm this for the graph matching problem in the experimental Section~\ref{sec:Experiments} below.
% \end{remark}

\myparagraph{Optimization Subproblems of Procedure~\ref{alg:factor-optimization}.} For each call of Procedure~\ref{alg:factor-optimization} one must find the best factor element in line~\ref{alg1:optimize-factor} and compute the best messages by solving~\eqref{eq:ReverseFactorOptimization}. The first subproblem is solved by explicitly scanning all elements of the factor for node-, edge- and label-factors. For optimizing over $\SM$, we use a min-cost-flow solver.
Solving~\eqref{eq:ReverseFactorOptimization} for all choices of factors and neighborhoods is possible through closed-form solutions or calling a mainimum cost flow solver and is described in the appendix.
%\textcolor{red}{
%in Table~\ref{table:ReparametrizationAdjustment}. It is worth noting that in all cases except those when $i=\SM$, we give an explicit computation formula. Otherwise the optimization (see~\eqref{eq:FlowToUnaryMessage-1} and~\eqref{eq:FlowToUnaryMessage-2}) reduces to a dual of the min-cost flow problem. {\color{red} We refer to the supplement for more details.} % TODO: do it.
%}

% There are efficient ways to solve the optimization problems~\eqref{eq:FactorOptimization} and~\eqref{eq:ReverseFactorOptimization} occurring in Algorithm~\ref{alg:factor-optimization}.
% To solve~\eqref{eq:FactorOptimization}, for unary, pairwise and label factors, we choose the minimal label by scanning all elements of the factor explicitly.
% For optimizing over $\SM$, we can use a min-cost-flow solver.
% Solving~\eqref{eq:ReverseFactorOptimization} for all choices of factors and neighborhoods is described in Table~\ref{table:ReparametrizationAdjustment}.
% One can derive correctness of all operations mechanically. We refer to the supplementary material for proofs of correctness.

\begin{table}
        \scalebox{0.9}{
   \centering
\small
\begin{tabular}{|gccc|}
%      \begin{tabular}{|ccccc|}
\hline
        Algorithm & Ordered set $I$ & $J_{r}(i)$ & $J_s(i)$ \\ \hline
        \textbf{GM-O} & $ \{u_1,\fdots,u_n\}_{<}$ & $\SN_{\BG}(i)_<$ & $\SN_{\BG}(i)_>$ \\
        \textbf{AMP-O} & $ \{u_1,\fdots,u_n,s_1,\fdots,s_{\abs{\SL}}\}_<$ & $\SN_{\BG}(i)_<$ & $\SN_{\BG}(i)_>$\\
        \textbf{AMCF-O} & $\{u_1,\fdots,u_n,\SM\}_{<}$ & $\SN_{\BG}(i)_< \cap \FE$ & $\SN_{\BG}(i)_>$ \\
        \textbf{AMP-C} & $ \{u_1,\fdots,u_n,s_1,\fdots,s_{\abs{\SL}}\}_<$ & $\SN_{\BG}(i)_<$ & $\SN_{\BG}(i)_>$\\
        \textbf{AMCF-C} & $\{u_1,\fdots,u_n,\SM,l_1,\fdots,l_{\abs{\SL}}\}_{<}$ & $\SN_{\BG}(i)_< \cap \FE$ & $\SN_{\BG}(i)_>$\\    
%         {GM-O} & $ \{u_1,\fdots,u_n\}_{<}$ & $\SN_{\BG}(i)_<$ & $\SN_{\BG}(i)_>$  \\
%         {AMP-O} & $ \{u_1,\fdots,u_n,s_1,\fdots,s_{\abs{\SL}}\}_<$ & $\SN_{\BG}(i)_<$ & $\SN_{\BG}(i)_>$ \\
%         {AMCF-O} & $\{u_1,\fdots,u_n,\SM\}_{<}$ & $\SN_{\BG}(i)_< \cap \FE$ & $\SN_{\BG}(i)_>$ \\
%         {AMP-C} & $ \{u_1,\fdots,u_n,s_1,\fdots,s_{\abs{\SL}}\}_<$ & $\SN_{\BG}(i)_<$ & $\SN_{\BG}(i)_>$ \\
%         {AMCF-C} & $\{u_1,\fdots,u_n,\SM,l_1,\fdots,l_{\abs{\SL}}\}_{<}$ & $\SN_{\BG}(i)_< \cap \FE$ & $\SN_{\BG}(i)_>$ \\    
\hline        
\end{tabular}

\caption{
        Input sets for specializations of Algorithm~\ref{alg:message-passing}.
        For algorithms with the suffix {\bf -I} the sets are the same as for those with {\bf -O}, but correspond to the inverse graph.
%         Factor ordering for Algorithm~\ref{alg:message-passing} 
%          for \textbf{GM}, \textbf{AMCF}, \textbf{AMP}.
%          Let the nodes of the graph matching be ordered as $\FV = {\{u_1,\ldots,u_n\}}_{<}$ and the labels as $\SL = {\{s_1,\ldots,s_{\abs{\SL}}\}}_{<}$.
%          The matching factor used in~\eqref{eq:local-polytope-matching-factor} is denoted by $\SM$.
%          The set $J_r(i)$ excludes $\SM$ whenever applicable. By this choice, we only optimize over $\SM$ once per iteration.
}
        }
\label{table:factor-order}
\end{table}

\myparagraph{Primal Rounding.}
Algorithm~\ref{alg:message-passing} only provides lower bounds to the original problem~\eqref{eq:GraphMatching}.
To obtain a primal solution one may ignore the edge potentials $\theta_{uv}$ and solve the resulting reparametrized bipartite matching problem~\eqref{eq:matching-constraints} with a minimum cost flow solver, as done in~\cite{HungarianBP}.
Empirically we found that it is better to interleave rounding and message passing, similarly as in TRWS~\cite{TRWSKolmogorov} and SRMP~\cite{SRMPKolmogorov}.
Assume we have already computed a primal integer solution $x^*_v$ for all $v<u$ and we want to compute $x^*_u$. 
To this end, between lines 4 and 5 of Algorithm~\ref{alg:message-passing} for $i=u$ we assign
\begin{equation}
         x_u^* \in \argmin_{x_u : x_u\neq x_v^* \forall v < u} \theta_u(x_u) + \sum_{v < u: uv \in \FE} \theta_{uv}(x_u, x_v^*)\,.
\end{equation} 

% Whenever we want to compute a primal solution $x^* \in X_{\FV}$,
% before receiving messages in Algorithm~\ref{alg:message-passing} for $u \in \FV$,
% we choose an arbitrary solution of
% \begin{equation}
%         x_u^* \in \argmin_{x_u : x_v^* \neq x_u \forall v < u} \theta_u(x_u) + \sum_{v < u: uv \in \FE} \theta_{uv}(x_u, x_v^*)\,.
% \end{equation} 

% \begin{remark}
%         In~\cite{SRMPKolmogorov} it was shown that using such an interleaved primal computation for~\eqref{eq:GraphicalModel} leads to an order of magnitude faster primal convergence for most problems.
% \end{remark}

\myparagraph{Time complexity}
        If $X_u = \SL$ $\forall u \in \FV$, time complexity per iteration is $O(\abs{\SL} \abs{\FV}+ \abs{\SL}^2 \abs{\FE})$ for \textbf{GM}. 
        For \textbf{AMP} we must add $ \abs{\SL}^3$ and for \textbf{AMCF} the time to solve~\eqref{eq:matching-constraints} (possible in $O(\SL^3)$).
Details and speedups are in the appendix.

\myparagraph{Higher Order Extensions.}
%\label{sec:Extensions}
Our approach is straightforwardly extendable to the graph matching problem with higher order factors, a special case being third order:
Let $\FT \subseteq{{\FV}\choose{3}}$ be a subset ot triplets of nodes and $\theta_{uvw} : X_{uvw} \rightarrow \R$ be corresponding triplet potentials.
The corresponding third order graph matching problem reads
\begin{align}
        \min_{x \in \mathcal{X}_{\FV}} & \sum_{u \in \FV} \theta_u(x_u) + \sum_{uv \in \FE} \theta_{uv}(x_{uv}) + \sum_{uvw \in \FT} \theta_{uvw}(x_{uvw})\nonumber \\
        \text{s.t. } & x_u \neq x_v \ \forall u\neq v\,. \label{eq:HigherOrderGraphMatching}
\end{align}
The associated \IPSLP\ can be constructed by including additional factors for all triplets in an analoguous fashion as in~\eqref{eq:LP:energy-min}, see e.g.~\cite{Werner10} for the corresponding relaxation.

For relaxations (R1) -- (R5) we use third order factors to enforce cycle inequalities, which we add in a cutting plane approach as in~\cite{FrustratedCyclesSontagEtAl}.
For this we set $\theta_{uvw} \equiv 0$ at the beginning.
By this construction~\eqref{eq:HigherOrderGraphMatching} is equivalent to~\eqref{eq:GraphMatching}, however the corresponding \IPSLP\ are not: Triplet potentials make the relaxation tighter.

%Except for~\cite{HungarianBP}, this higher order extension seems hard to implement in~\cite{CombinatorialOptimizationMaxProductDuchi,CoveringTreesLowerBoundQuadraticAssignmentJarkony,MRFSemidefiniteTorr,GraphMatchingDDTorresaniEtAl,SpectralTechniqueAssignmentLeordeanu,ProbabilisticSubgraphMatchingSchellewald,GraduatedAssignmentGold}.

\section{Experiments}
\label{sec:Experiments}
%\subsection{Algorithms}
\myparagraph{Algorithms.}
We compare against the two Lagrangean decomposition based solvers~\cite{GraphMatchingDDTorresaniEtAl,HungarianBP} described in Section~\ref{sec:intro}.

\begin{itemize}
\item The dual decomposition solver~\textbf{DD}~\cite{GraphMatchingDDTorresaniEtAl}.
        %The algorithm connects complicated subproblems (max-flow, min-cost-flow and ``local subproblems'') by Lagrangian variables and uses subgradient ascent to optimize the resulting dual lower bound.
        We use local subproblems containing $4$ nodes. Note that the comparison in~\cite{HungarianBP} was made with subproblems of size $3$, hence \textbf{DD}'s relaxation was weaker there.
\item ``Hungarian belief propagation''\textbf{HBP}~\cite{HungarianBP}.
        %which is also a dual ascent addressing the same relaxation and decomposition as {\bf AMCF}. As noted in Section~\ref{sec:graph-matching-algs}, the main difference consists in a message passing schedule.
        In~\cite{HungarianBP} a branch and bound solver is used on top of the dual ascent solver. For a fair comparison our reimplementation uses only the dual ascent component.
               % and tightens it with cycling constraints in the same manner as in our solvers.
                As for \textbf{AMP} and \textbf{AMCF}, we append to \textbf{HBP} the suffixes \textbf{-O} and \textbf{-C} to denote the relaxations we let \textbf{HBP} run on.
%        In~\cite{HungarianBP} the message passing routine was embedded in a branch and bound solver.  For a fair comparison our reimplementation only uses the message passing component and tightens it with cycling constraints in the same manner as our solvers.
\end{itemize}
According to~\cite{GraphMatchingDDTorresaniEtAl,HungarianBP}, these two algorithms outperformed competitors~\cite{GraduatedAssignmentGold,SpectralTechniqueAssignmentLeordeanu,BalancedGraphMatchingCour,CombinatorialOptimizationMaxProductDuchi,CoveringTreesLowerBoundQuadraticAssignmentJarkony,GlobalSolutionCorrespondenceJoao,MRFSemidefiniteTorr,ProbabilisticSubgraphMatchingSchellewald,IntegerFixedPointGraphMatching,RandomWalksForGraphMatching,FactorizedGraphMatching,LocalSparseMatching} at the time of their publication, hence we do not compare against the latter ones.

We set a hard threshold of $1000$ iterations for each algorithm, exiting earlier when the primal/dual gap vanishes or no dual progress was observed anymore.
We compute primal solutions every $5$-th iteration in our algorithms. For \textbf{GM}, \textbf{AMP}, \textbf{AMCF} and \textbf{HBP} we use the tightening extension discussed in Section~\ref{sec:graph-matching-algs} to improve the dual lower bound. We tighten our relaxation whenever no dual progress occurs.

%\subsection{Datasets}
\myparagraph{Datasets.}

\begin{table} 
        \small
\centering% %        
\begin{tabular}{|gcccc|} %      
\hline %                
\texttt{dataset} & \#I & \#V & \#L & C \\ \hline %                
\texttt{house} & 105 & 30 & 30 & dense \\ %                
\texttt{hotel} & 105 & 30 & 30 & dense \\ %                
\texttt{car} & 30 & 19-49 & 19-49 & dense \\ %                
\texttt{motor} & 20 & 15-52 & 15-52 & dense \\ %                
\texttt{graph flow} & 6 & 60-126 & 2-126 & sparse \\ %                
\texttt{worms} & 30 & $\leq600$ & 20-60 & sparse \\ %      
\hline %   
\end{tabular} %   
\caption{ %           
        Dataset description. \#I denotes number of instances, \#V the number of nodes $\abs{\FV}$, \#L the number $\abs{X_u}$ of labels a node $u \in \FV$ can be matched to and C the connectivity of the graph. %        
} %   
\label{table:DatasetDescription} %
\end{table}

\begin{figure*}%
\centering%
\begin{minipage}{0.3\textwidth}%
\vspace{0pt}%
  \includegraphics[width=1.5\textwidth]{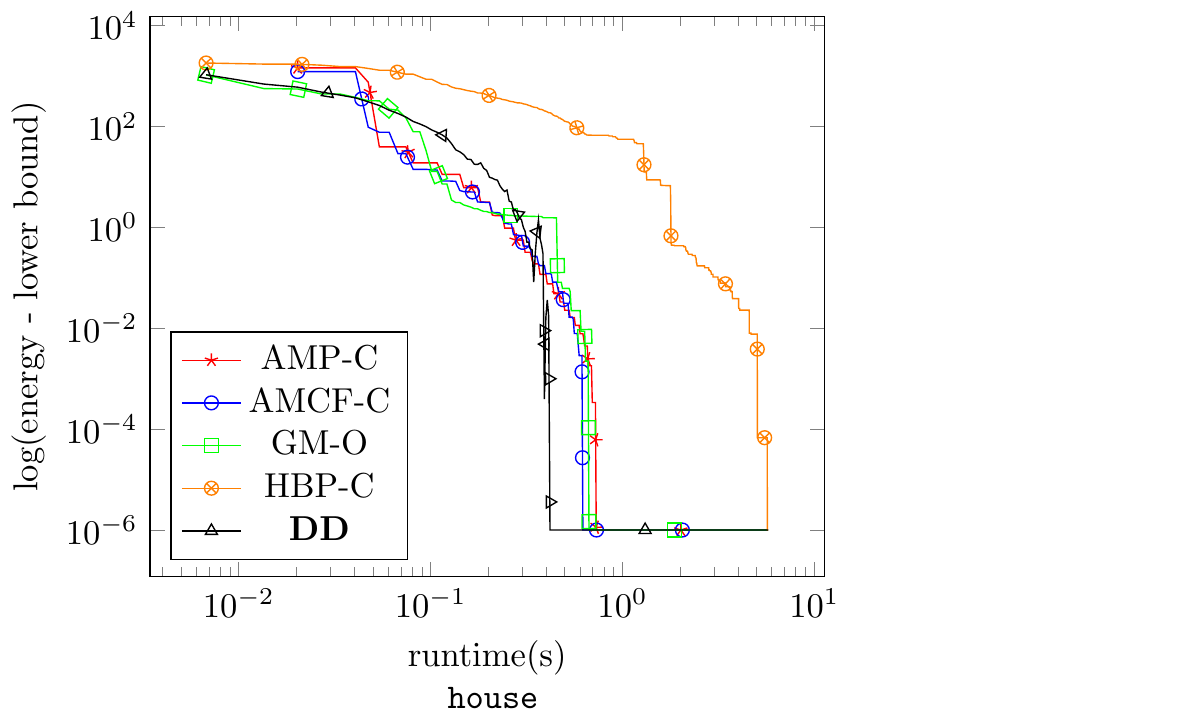}%
\end{minipage}%
\hspace*{0.8cm}%
\begin{minipage}{0.3\textwidth}%
\vspace{0pt}%
  \includegraphics[width=1.5\textwidth]{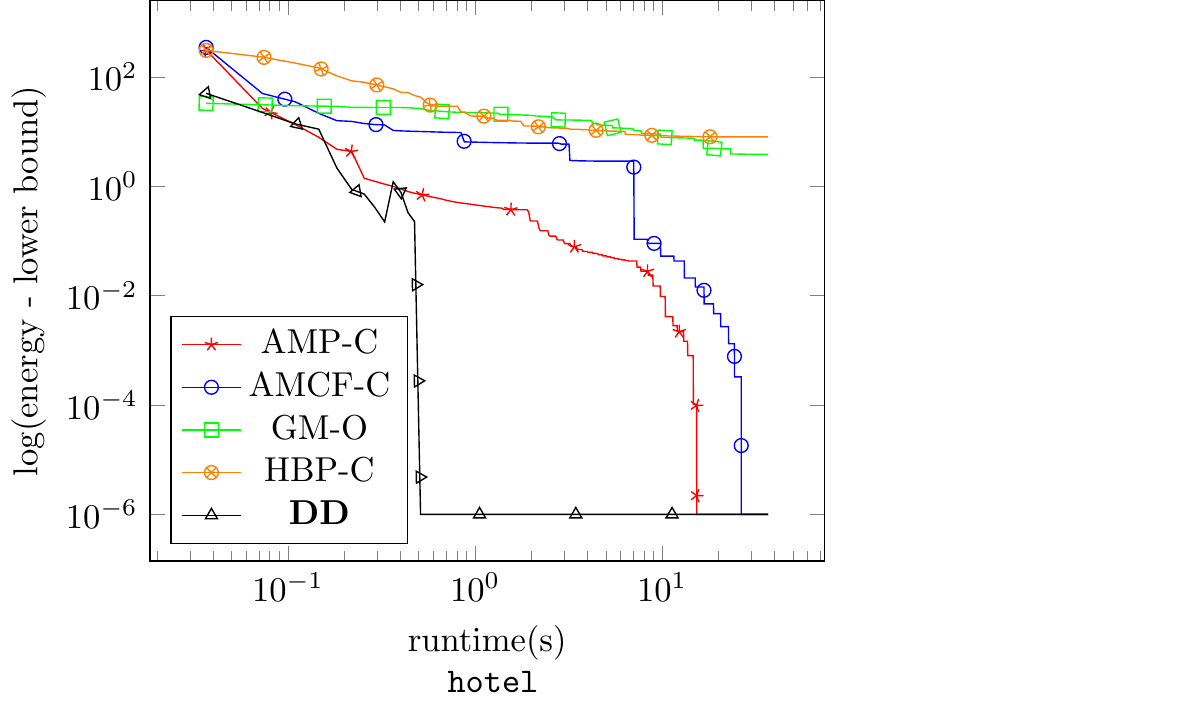}%
\end{minipage}%
\hspace*{0.4cm}%
\begin{minipage}{0.3\textwidth}%
\vspace{0pt}%
  \includegraphics[width=1.5\textwidth]{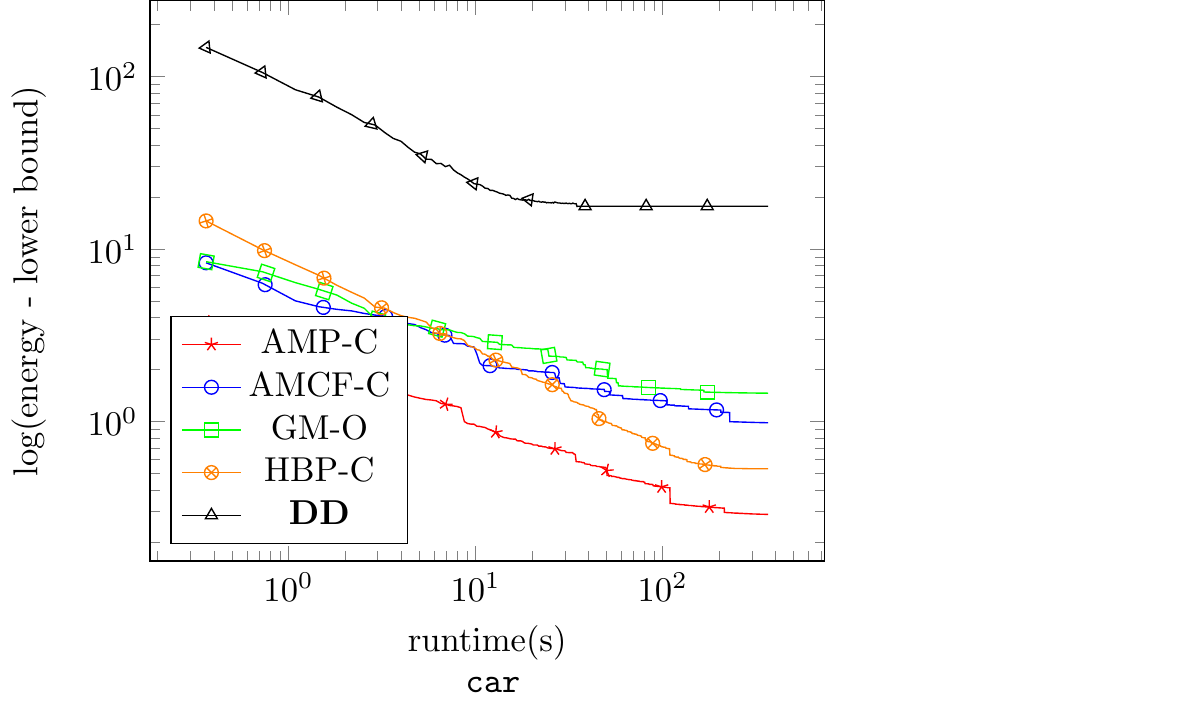}
\end{minipage}%
\\
\begin{minipage}{0.3\textwidth}%
\vspace{0pt}%
  \includegraphics[width=1.5\textwidth]{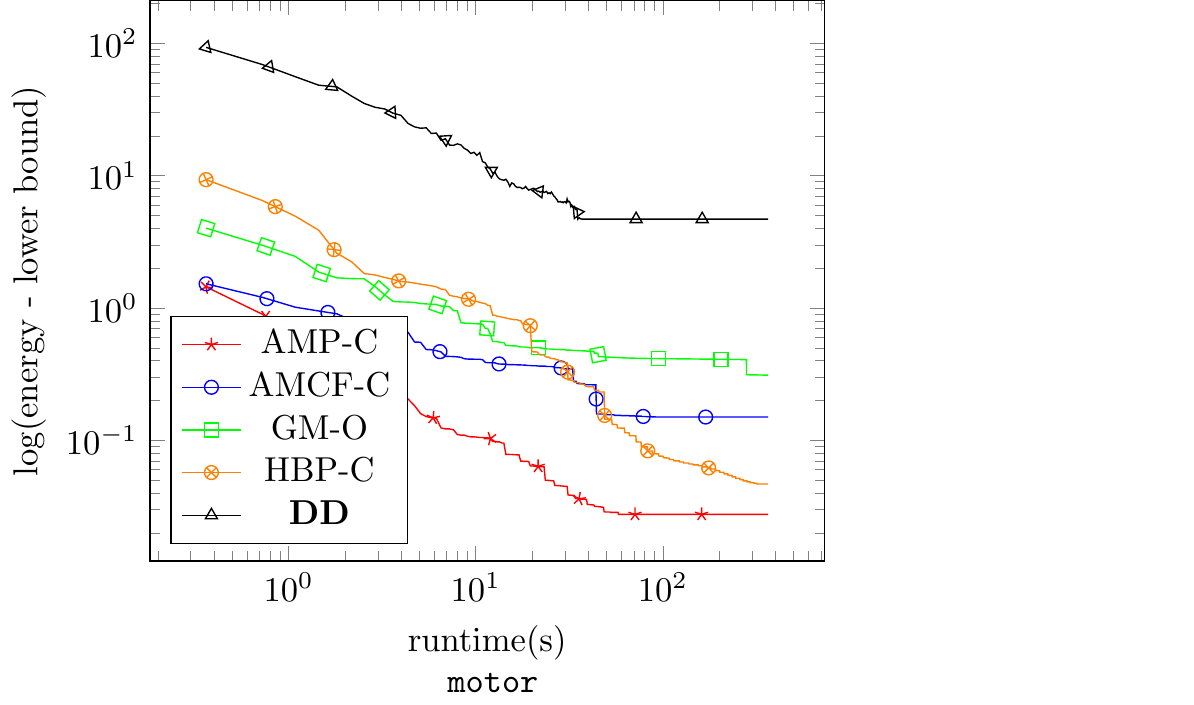}%
\end{minipage}%
\hspace*{0.8cm}%
\begin{minipage}{0.3\textwidth}%
\vspace{0pt}%
  \includegraphics[width=1.5\textwidth]{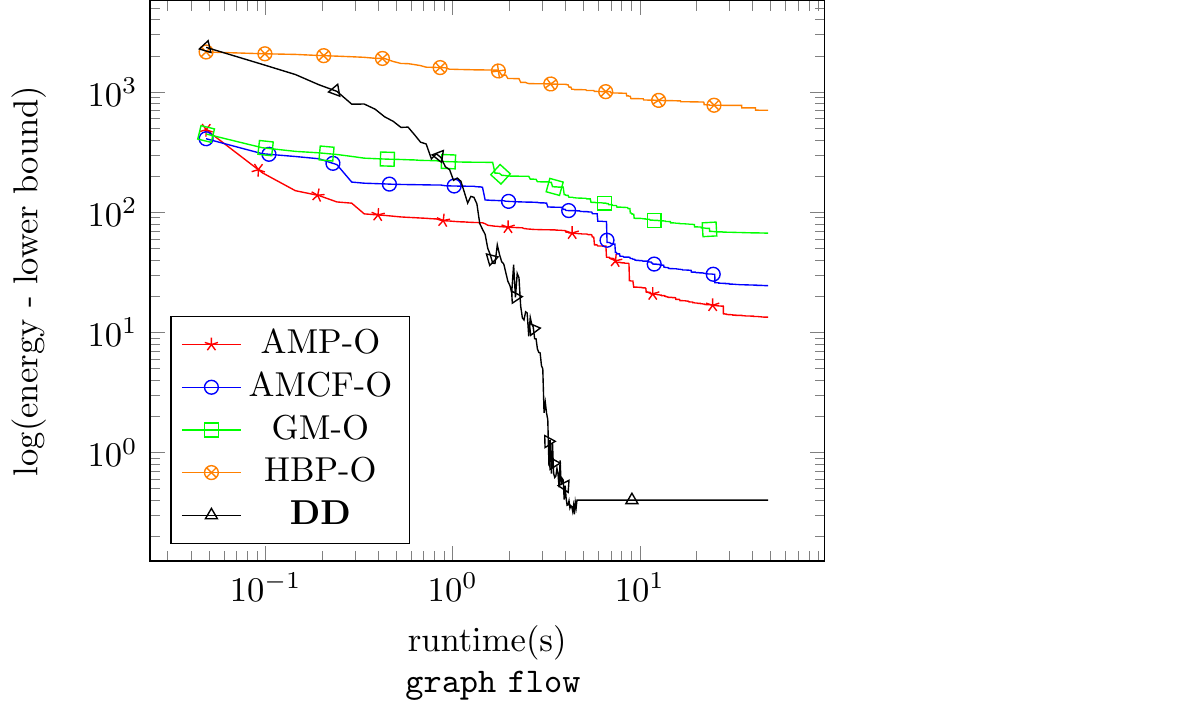}%
\end{minipage}%
\hspace*{0.4cm}%
\begin{minipage}{0.3\textwidth}%
\vspace{0pt}%
  \includegraphics[width=1.5\textwidth]{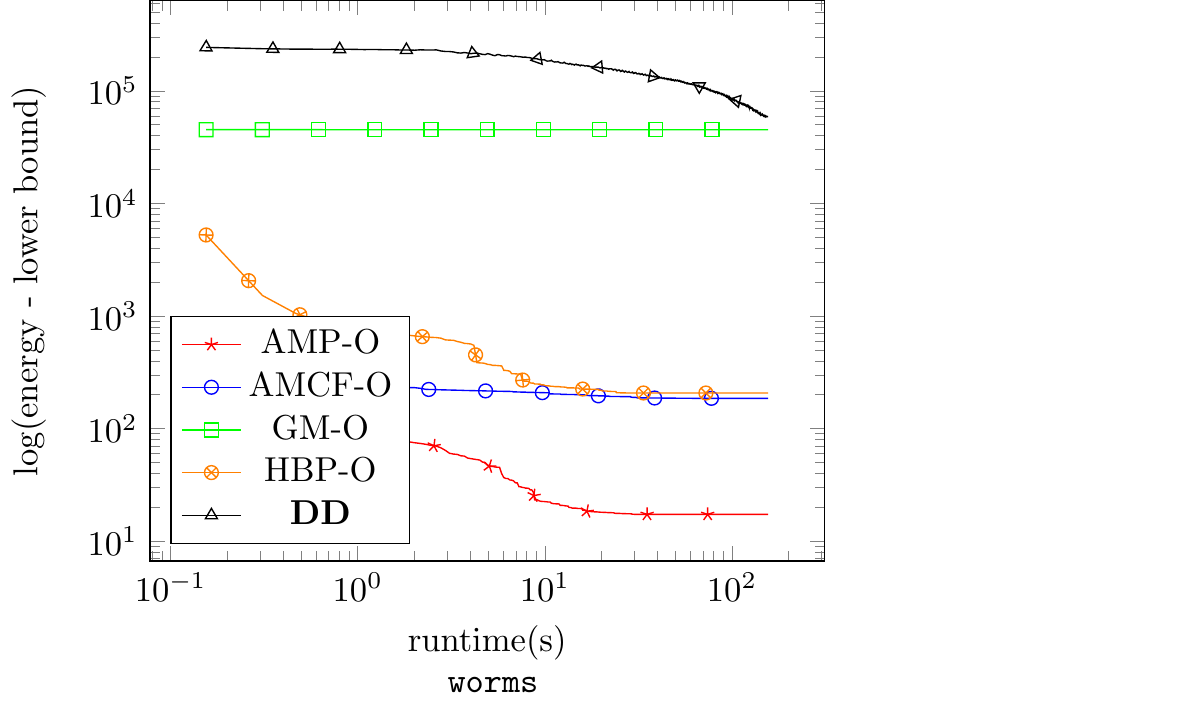}%
\end{minipage}%
  \caption{
  Plots showing convergence over time for \texttt{house}, \texttt{hotel}, \texttt{car}, \texttt{motor}, \texttt{graph flow} and \texttt{worms} datasets. 
        Values denote $\log(\text{upper bound} - \text{lower bound})$. 
  Values are averaged over all instances of the dataset.
  The x-axis and y-axis are logarithmic.
  }
\label{fig:DatasetRuntimePlot}
\end{figure*}

We have compared on six datasets:
\begin{itemize}
        \item \texttt{house}~\cite{HouseDataset} and \texttt{hotel}~\cite{HotelDataset} with costs as in~\cite{GraphMatchingDDTorresaniEtAl}. The task is to find matching feature points between images capturing an object from different viewpoints.
        \item \texttt{car} and \texttt{motor}, both used in~\cite{UnsupervisedLearningForGraphMatching}, containing pairs of cars and motorbikes with keypoints to be matched. The images are taken from the VOC PASCAL 2007 challenge~\cite{VocPascal}.
                Costs are computed from features as in~\cite{UnsupervisedLearningForGraphMatching}.
        \item The \texttt{graph flow} dataset~\cite{GraphFlowDataset} comes from a tracking problem with large displacements~\cite{GraphFlow}.
                Keypoints in frames of RGB-D images obtained by a Kinect camera~\cite{Kinect} are matched. The depth information provided by the Kinect camera is taken into account when computing the potentials $\theta$.
        \item The \texttt{worms} dataset~\cite{KainMueller2014active} from bioimaging.
                The goal is to annotate nuclei of C. elegans, a famous model organism used in developmental biology, by finding the corresponding nuclei in a precomputed model.
                The instances of {\tt worms} are, to our knowledge, the largest graph matching datasets ever investigated in the literature.
%         \item The novel and not yet published large scale dataset \texttt{worms}~\cite{KainMueller2014active} comes from a biological application.
%                 The goal is to annotate nuclei of C. elegans, a famous model organism used in developmental biology, by finding the corresponding nuclei in a precomputed model.
%                 The instances of worms are, to our knowledge, the largest graph matching datasets ever investigated in the literature.
\end{itemize}
Whereas the first $5$ datasets are publicly available, the \texttt{worms} dataset was kindly given to us by the authors of~\cite{KainMueller2014active} and will be published.
%The datasets consist of $105$, $105$, $30$, $20$, $6$ and $30$ instances respectively, altogether 296 instances.
A summary of dataset characteristics can be found in Table~\ref{table:DatasetDescription}.
Previous computational studies concentrated on small-scale problems having up to $60$ nodes and labels.
We have included the {\tt worms} dataset with up to $500$ nodes and $\abs{\SL} = 1500$ labels.  

\myparagraph{Results.}
Fig.~\ref{fig:DatasetRuntimePlot} shows performance of the algorithms on all $6$ considered datasets.
Among all variants {\bf -O}, {\bf-I}, {\bf -C} corresponding respectively to the original, inverse and coupled formulations we plotted only the best one. As expected, for dense graphs (datasets {\tt house}, {\tt hotel}, {\tt car}, {\tt motor}) the variant {\bf -C} with coupling provided most robust convergence, being similar to the best of {\bf -O} either {\bf -I} and therefore is presented on the plots.
%For comparison, we plot the different algorithm variants in {\color{red} Fig.~\ref{}}.
For sparse graphs, the inverse representation becomes too expensive, as the inverse edge set $\FE'$ may be dense even though $\FE$ is sparse in~\eqref{eq:InverseGraphMatching}.
Therefore we stick to the original problem {\bf -O}.

\begin{itemize}
\item \texttt{hotel} and \texttt{house} are easy datasets, and many instances needed $< 5$ iterations for convergence. {\bf AMP}, {\bf AMCF} and {\bf DD} were able to solve all instances to optimality within few seconds or even faster. However, {\bf DD} is the fastest method for this data.
\item \texttt{car} and \texttt{motor} were already harder and the $1000$ iteration limit did not allow to ascertain optimality for all instances. {\bf AMP} significantly outperforms its competitors, {\bf DD} is significantly slower than the rest, whereas other algorithms show comparable results. 
\item on \texttt{worms} again {\bf AMP} significantly outperforms its competitors, {\bf AMCF} and {\bf HBP} converge to similar duality gap, although {\bf AMCF} does it one-two orders of magnitude faster, {\bf GM} and {\bf DD} return results which are hardly competitive.
\item \texttt{graph flow} is the only dataset, where {\bf DD} clearly overcomes all competitors, followed by {\bf AMP}. 
% Moreover {\bf DD} is the only algorithm that attains the global optimum on all $6$ instances.
        We attribute it to \textbf{DD}'s tighter relaxation, (its "local" subproblems contain $4$ variables, whereas our subproblems have at most 3 variables after tightening.) % TODO: can we run DD with only 3 vars in local subproblems?
\end{itemize}

\myparagraph{Insights and Conclusions}
\begin{itemize}
        \item {\bf AMP} shows overall best performance for both small dense and large sparse datasets. It is the best anytime solver: it has the best performance in the first iterations. This is beneficial (i) if the run-time is limited or (ii) in branch-and-bound procedures, where a good early progress helps to efficiently eliminate non-optimal branches fast.
 \item Although {\bf AMP}, {\bf AMCF} and {\bf HBP} address equivalent relaxations (having the same maximal dual value) their convergence speed is different. {\bf AMCF} and {\bf HBP} are generally slower than {\bf AMP}, which we attribute to the suboptimal redistribution of the costs by the min-cost-flow factors $\{\SM\}$ when maximizing messages in~\eqref{eq:ReverseFactorOptimization}.%, see more discussion~\cite{ConvergentMessagePassingNIPS}.
 \item {\bf DD}'s relatively good performance is probably due to the large subproblems used by this method. First, this decreases the number of dual variables, which accelerates bound convergence; second, this makes the relaxation tighter, which decreases the duality gap. We attribute slow convergence of {\bf DD} to the subgradient method.
 \item Summarizing, larger subproblems are profitable for the sub-gradient method, but not for message passing. 
 \item {\bf AMCF} outperforms {\bf HBP} due to better message scheduling.
 \item We attribute the inferior performance of {\bf GM} mostly to the weakest relaxation it optimizes. Even under this condition, due to a good message scheduling and fast message passing it outperforms {\bf DD} and {\bf HBP} on several datasets. 
\end{itemize}

A detailed evaluation of all instances is in the appendix.

%{\color{gray}
%\myparagraph{Conclusion}
%The overall best algorithm is \textbf{AMP}. While it is outperformed in one dataset by subgradient-based algorithm \textbf{DD}, it is overall more robust: It does not degrade by much on some of the problems, as does \textbf{DD}. It also has an excellent anytime performance: Stopping \textbf{AMP} early will still give decent primal and dual results.
%Another interesting insight concerns the structure of the decomposition: 
%Unlike \textbf{AMCF}, \textbf{HBP} and \textbf{DD}, method \textbf{AMP} does not use any complex combinatorial subproblem.
%It is widely known that subgradient based methods like \textbf{DD} profit from having only a few large subproblems into which the overall problem is decomposed, as this decreases the communication effort between the subproblems.
%For message-passing, this does not hold: 
%Algorithms \textbf{AMCF} and \textbf{HBP} both use a min-cost-flow subproblem, yet are outperformed by \textbf{AMP}, which more tightly integrates sending message via Lagrangean variables and optimizing subproblems for improving the dual bound.
%Also in the primal domain, \textbf{AMP} and \textbf{AMCF} interleave primal rounding and message passing as much as possible, in contrast to \textbf{HBP}, which only uses the min-cost-flow factor to find a feasible labeling. This also leads to improves primal solution quality.
%}
%%%%%%%%%%%%%%%%%%%%%%%%%%%%%%%%%%%%%%%%%%%%%%%%%%%%%%%%%%%%%%%%%%%%%%%%%%%

\section{Acknowledgments}
The authors would like to thank Vladimir Kolmogorov for helpful discussions.
This work is partially funded by the European Research Council under the European Unions Seventh Framework Programme (FP7/2007-2013)/ERC grant agreement no 616160.

{\small
\bibliographystyle{../style.files/ieee}
\bibliography{../literatur}
}
\clearpage
\onecolumn
\section{Supplementary Material}

\subsection{Proofs}
\myparagraph{Proof of Proposition~\ref{prop:RelaxationRelation}}

\subsection{Optimization Subproblems of Procedure~\ref{alg:factor-optimization}}
In Procedure~\ref{alg:factor-optimization} the two problem in lines~\ref{alg1:optimize-factor} and~\ref{alg1:optimize-messages} must be solved.
Solution of the optimization problem in line~\ref{alg1:optimize-factor} was discussed in the main part of the paper.
Therefore, it only remains to show how to carry optimization of the problem in line~\ref{alg1:optimize-messages} efficiently for all cases that can occur.
This is shown in Table~\ref{table:ReparametrizationAdjustment}.

Checking validity of the operations in Table~\ref{table:ReparametrizationAdjustment} for $i = u \in \FV$ and $i = uv \in \FE$ is straightforward.
For $i = \SM$ and $J = \FV$, we prove correctness below. Correctness for $i = \SM$ and $J = \SL$ is analoguous.
\begin{lemma}
The reparametrization adjustment problem~\eqref{eq:ReverseFactorOptimization} for $i = \SM$ and $J = \FV$ is given by~\eqref{eq:FlowToUnaryMessage-1}.
Moreover it is the dual of a minimum cost network flow problem.
\end{lemma}
\begin{proof}
  Recall from network flow theory~\cite{AhujaMagnantiOrlinNetworkFlows}, that $x_{u}^* \in X_u$ is optimal for cost $\tilde\theta$, iff
  $\exists \pi \in \R^{\abs{\FV}}, \psi \in \R^{\abs{\SL}}$ such that
  \begin{equation*}
  \tilde\theta_{u}(x_u) - \pi(u) + \psi(x_u) \begin{cases} \leq 0, & x_u^*(x_u) = 1\\ \geq 0, & x_{u}^*(x_u) = 0 \end{cases} \forall u \in \FV, x_u \in X_u\,.
  \end{equation*}
Consider the primal/dual pair
\begin{equation}
   \label{eq:PrimalDualNetworkFlowReverseOptimization}
   \begin{array}{lr}
     \min_{(\mu_u)_{u \in \FV}} \sum_{u \in \FV} \la \tilde\theta_u, \mu_u \ra 
     & \max_{\Delta,\pi,\psi} \sum_{u \in \FV} \la \Delta_u, \delta_u \ra \\
      \forall u \in \FV, \sum_{x_u \in X_u} \mu_u(x_u) = 0  
      & \pi(u) \in \R \\
      \forall l \in \SL, - \sum_{u : l \in X_u} \mu_u(l) = 0 
      & \psi(l) \in \R \\
      \mu_u(x_u) \begin{cases} \geq \delta_u(x_u),& x_u \neq x_u^* \\ \leq \delta_u(x_u),& x_u = x_u^* \end{cases} 
        & \Delta_u(x_u) \in \begin{cases} \geq 0, & x_u^*  = x_u \\ \leq 0, & x_u^* \neq x_u \end{cases} \\
         \mu_u(x_u) \in \R 
         & \tilde\theta_u(x_u) + \Delta_u(x_u) + \pi(u) - \psi(x_u) = 0 \\
   \end{array}
\end{equation}
  On the right side, the adjustment problem~\eqref{alg1:optimize-messages} is written down explicitly. 
  The left hand side is a minimum cost flow problem, hence the second part of the claim is proven.

  The last equality above on the right hand side ensures that $\Delta_{u}(x_u) = - \tilde\theta_u(x_u) - \pi(u) + \psi(x_u)$.
  Substituting this everywhere on the right hand side of~\eqref{eq:PrimalDualNetworkFlowReverseOptimization} gives
  \begin{equation}
    \begin{array}{rl}
      \max_{\pi,\psi} & \sum_{u \in \FV} \pi(u) \cdot \left(\sum_{x_u \in \SX_u} \delta_u(x_u) \right) + \sum_{l \in \SL} \psi(l) \cdot \left( \sum_{u \in \FV: l \in X_u} \delta_u(l) \right)\\
       \text{s.t.} &\pi(u) \in \R \\
       &\psi(l) \in \R \\
         &   \tilde\theta_u(x_u) + \pi(u) - \psi(x_u) \begin{cases} \leq 0, & x_u^* = x_u \\ \geq 0, & x_u^* \neq x_u \end{cases} \\
    \end{array}
  \end{equation}
  This form matches the format given in~\eqref{eq:FlowToUnaryMessage-1}.
\end{proof}

\subsection{Time complexity}
The time complexity of running one iteration of message passing is essentially the time to run all required invocations of Algorithm~\ref{alg:factor-optimization} via the routines described in Table~\ref{table:ReparametrizationAdjustment}.
Total runtime per iteration for the various algorithms we have proposed can be found in Table~\ref{table:time-complexity}.
We assume that $X_u = \SL$ $\forall u \in \FV$. 
In sparse assignment problems, where this is not the case, run-time decreases according to sparsity.

If we hold the unary potentials $\theta_u$, $u \in \FV$ in a heap, we can support operation $\min_{s' \in J \cap X_u} \theta_u(s')$ which is required in the third line in Table~\ref{table:ReparametrizationAdjustment} in time $\log(\abs{\SL})$, since either $J \cap X_u = X_u$ (sending) or $\abs{J \cap X_u} = 1$ (receiving).

Hence, all our algorithms scale to realistic problem sizes.

\subsection{Detailed Experimental Evaluation}
Plots showing lower bound and primal solution energy per over time can be seen in Figure~\ref{fig:DatasetRuntimePlot}.

In Table~\ref{table:DatasetResults} dataset statistics are given together with final upper and lower bound as well as runtime averaged over all instances in specific datasets are given.

A per-instance evaluation of all considered algorithms can be found in Table~\ref{Table:InstanceResults}.

\onecolumn
        \centering

\begin{table*}
\begin{tabular}{|ggc|}
\hline
        \multicolumn{2}{|g}{Algorithm~\ref{alg:factor-optimization} input} & \multirow{2}{*}{Solution  $\Delta^*_{(i,j)}$ $\forall j \in J$ of~\eqref{eq:ReverseFactorOptimization}} \\ 
$i \in \BV$ & $J \subseteq \SN_{\BG}(i)$ & \\ \hline
                && $\Delta^*_{(u,uv)} = \nicefrac{(\theta_u - \min_{x_u \in X_u} \theta_u(x_u))}{\abs{J}}$ $\forall uv \in \FE \cap J$ \\
                && $\Delta^*_{(u,\SM)} = \nicefrac{(\theta_u - \min_{x_u \in X_u} \theta_u(x_u))}{\abs{J}} $ \\ 
        \multirow{-3}{*}{$i = u \in \FV$}
        & \multirow{-3}{*}{$J \subseteq \FE \cup \{\SM\} \cup \SL $} 
                & $\Delta^*_{(u,s)} = \nicefrac{(\theta_u(s) - \min_{s' \notin J\cap X_u} \theta_u(s'))}{\abs{J}} $ $\forall s \in X_u \cap J$\\ \hline
        $i = uv \in \FE$ & $J = \{u\}, u \in \FV$ & $\Delta^*_{(uv,u)}(x_u) = \min\limits_{x_v \in \SX_v} \{\theta_{uv}(x_u,x_v)\} - \min_{x_{uv} \in X_{uv}}\{ \theta_{uv}(x_{uv}) \}$ \\ \hline
$i = \SM$ & $J = \FV$ & 
\parbox{0.6\textwidth}{
\begin{equation}
\label{eq:FlowToUnaryMessage-1}
\begin{array}{rl}
\multicolumn{2}{c}{ \Delta^*_u(x_u) = -\tilde\theta_u(x_u) - \pi^*(u) + \psi^*(x_u) } \\
  (\pi^*,\psi^*) \in \argmax_{\pi,\psi} & \begin{matrix}\sum_{u \in \FV} \pi(u) \cdot \left(\sum_{x_u \in \SX_u} \delta_u(x_u) \right) \\ + \sum_{l \in \SL} \psi(l) \cdot \left( \sum_{u \in \FV: l \in X_u} \delta_u(l) \right) \end{matrix}\\
       \text{s.t.}
         &   \tilde\theta_u(x_u) + \pi(u) - \psi(x_u) \begin{cases} \leq 0, & x_u^* = x_u \\ \geq 0, & x_u^* \neq x_u \end{cases} \\
\end{array}
\end{equation}
}
\\ \hline
$i = \SM$ & $J = \SL$ & 
\parbox{0.6\textwidth}{
\begin{equation}
\begin{array}{rl}
  \multicolumn{2}{c}{ \Delta^*_s(u) = -\tilde\theta_u(s) - \pi^*(u) + \psi^*(s) } \\
  (\pi^*,\psi^*) \in \argmax_{\pi,\psi} & \begin{matrix}\sum_{u \in \FV} \pi(u) \cdot \left(\sum_{x_u \in \SX_u} \delta_{x_u}(u) \right) \\ + \sum_{l \in \SL} \psi(l) \cdot \left( \sum_{u \in \FV: l \in X_u} \delta_l(u) \right) \end{matrix}\\
       \text{s.t.}
         &   \tilde\theta_u(x_u) + \pi(u) - \psi(x_u) \begin{cases} \leq 0, & x_u^* = x_u \\ \geq 0, & x_u^* \neq x_u \end{cases} \\
%\pi^* \in \argmax_{\pi \in \R^{\abs{\FV\dcup \SL}}}  & \sum_{u \in \FV} \pi(u) - \sum_{l \in \SL} \pi(l) \\
%        & \pi(u) - \pi(l) \leq \theta_{\SM}(u,l) \\
%        \multicolumn{2}{l}{ \Delta^*_{(\SM,s)}(u) =  \theta_{\SM}(u,s) - \pi^*(u) + \pi^*(s)\ \ \forall s \in \SL } \\
\end{array}
\label{eq:FlowToUnaryMessage-2}
\end{equation}
}
\\ \hline
\end{tabular}
\caption{Message computation problems~\eqref{eq:ReverseFactorOptimization}}
\label{table:ReparametrizationAdjustment}
\end{table*}

\begin{table*}
        \begin{tabular}{|gc|}
                \hline Algorithm & Time complexity \\ \hline
                \textbf{GM} & $O(\abs{\SL} \cdot \abs{\FV} + \abs{\SL}^2 \cdot \abs{\hat\FE})$\\
                \textbf{AMP}& $O(\abs{\SL} \cdot \abs{\FV} + \abs{\SL}^2 \cdot \abs{\FE} + \abs{\SL}^3)$ \\ 
                \textbf{AMP}${}^{\dagger}$& $O(\abs{\SL} \cdot \abs{\FV} + \abs{\SL}^2 \cdot \abs{\FE} + \abs{\SL}^2 \cdot \log \abs{\SL})$ \\ 
                \textbf{AMCF} & $O(\abs{\SL} \cdot \abs{\FV} + \abs{\SL}^2 \cdot \abs{\FE} + MCF(\abs{\FV}, \abs{\FV}^2)$ \\ \hline
        \end{tabular}
        \caption{
                Time complexity per iteration for the three proposed algorithms.
                %Note that $\FE$ may be larger for $\textbf{GM}$ than for $\textbf{AMP}$ and $\textbf{AMCF}$ due to additional pairwise potentials needed to ensure the uniqueness constraint in Section~\ref{sec:relaxation}. 
                $MCF(n,m)$ is the time to solve a min-cost-flow problem on a graph with $n$ nodes and $m$ edges: 
                Orlin's algorithm has time complexity $O(m^2 \log n + m  \log^2 n)$~\cite{AhujaMagnantiOrlinNetworkFlows}.
                \textbf{AMP}${}^{\dagger}$ stores reparametrized unary costs in a heap to accelerate computation of the messages between label factors and unaries.
        }
        \label{table:time-complexity}
\end{table*}

%\begin{figure*}%
%\centering%
%\begin{minipage}{0.3\textwidth}%
%\vspace{0pt}%
%  \includegraphics[width=1.5\textwidth]{plots/graph_matching_dataset_plot_house}%
%\end{minipage}%
%\hspace*{0.8cm}%
%\begin{minipage}{0.3\textwidth}%
%\vspace{0pt}%
%  \includegraphics[width=1.5\textwidth]{plots/graph_matching_dataset_plot_hotel}%
%\end{minipage}%
%\hspace*{0.4cm}%
%\begin{minipage}{0.3\textwidth}%
%\vspace{0pt}%
%  \includegraphics[width=1.5\textwidth]{plots/graph_matching_dataset_plot_car}
%\end{minipage}%
%\\
%\begin{minipage}{0.3\textwidth}%
%\vspace{0pt}%
%  \includegraphics[width=1.5\textwidth]{plots/graph_matching_dataset_plot_motor}%
%\end{minipage}%
%\hspace*{0.8cm}%
%\begin{minipage}{0.3\textwidth}%
%\vspace{0pt}%
%  \includegraphics[width=1.5\textwidth]{plots/graph_matching_dataset_plot_Hassan}%
%\end{minipage}%
%\hspace*{0.4cm}%
%\begin{minipage}{0.3\textwidth}%
%\vspace{0pt}%
%  \includegraphics[width=1.5\textwidth]{plots/graph_matching_dataset_plot_worms}%
%\end{minipage}%
% 
%  \caption{Per-iterations plots for \texttt{house}, \texttt{hotel}, \texttt{car}, \texttt{motor}, \texttt{graph flow} and \texttt{worms} datasets. 
% Continuous lines denote dual lower bounds and dashed ones primal energies. 
%  Values are averaged over all instances of the dataset.
%  The x-axis is logarithmic.
%  }
%\label{fig:DatasetIterationPlot}
%\end{figure*}

\begin{figure*}%
\centering%
\begin{minipage}{0.3\textwidth}%
\vspace{0pt}%
  \includegraphics[width=1.5\textwidth]{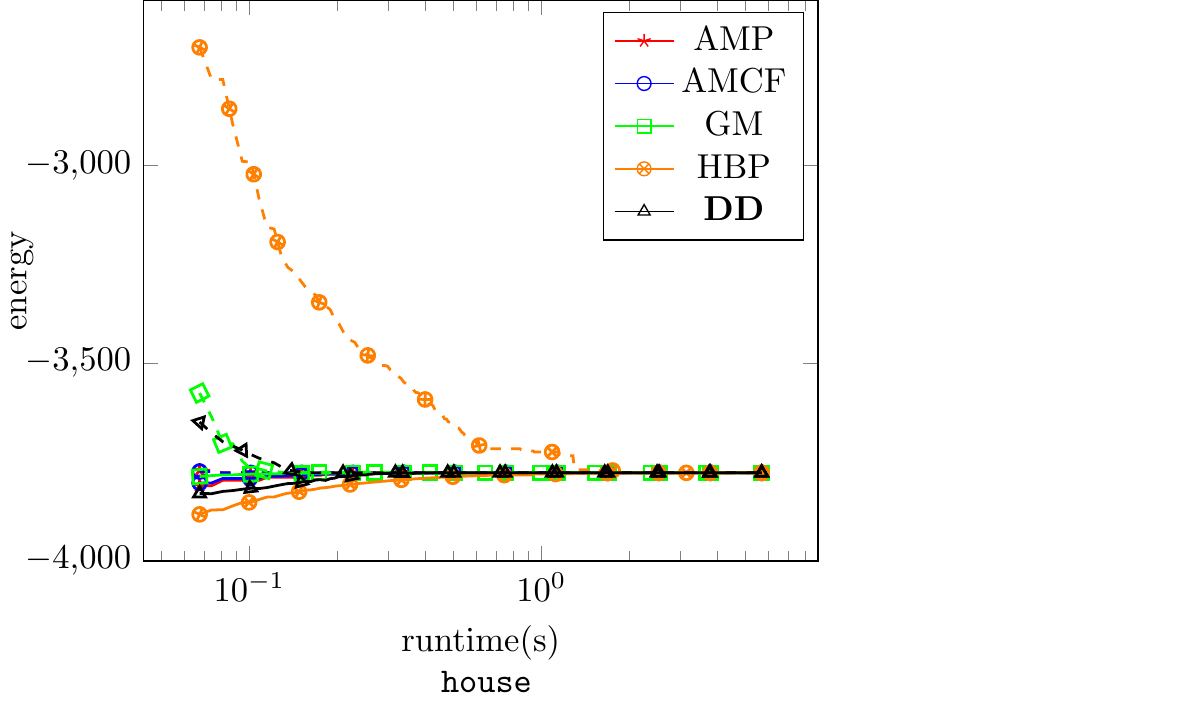}%
\end{minipage}%
\hspace*{0.8cm}%
\begin{minipage}{0.3\textwidth}%
\vspace{0pt}%
  \includegraphics[width=1.5\textwidth]{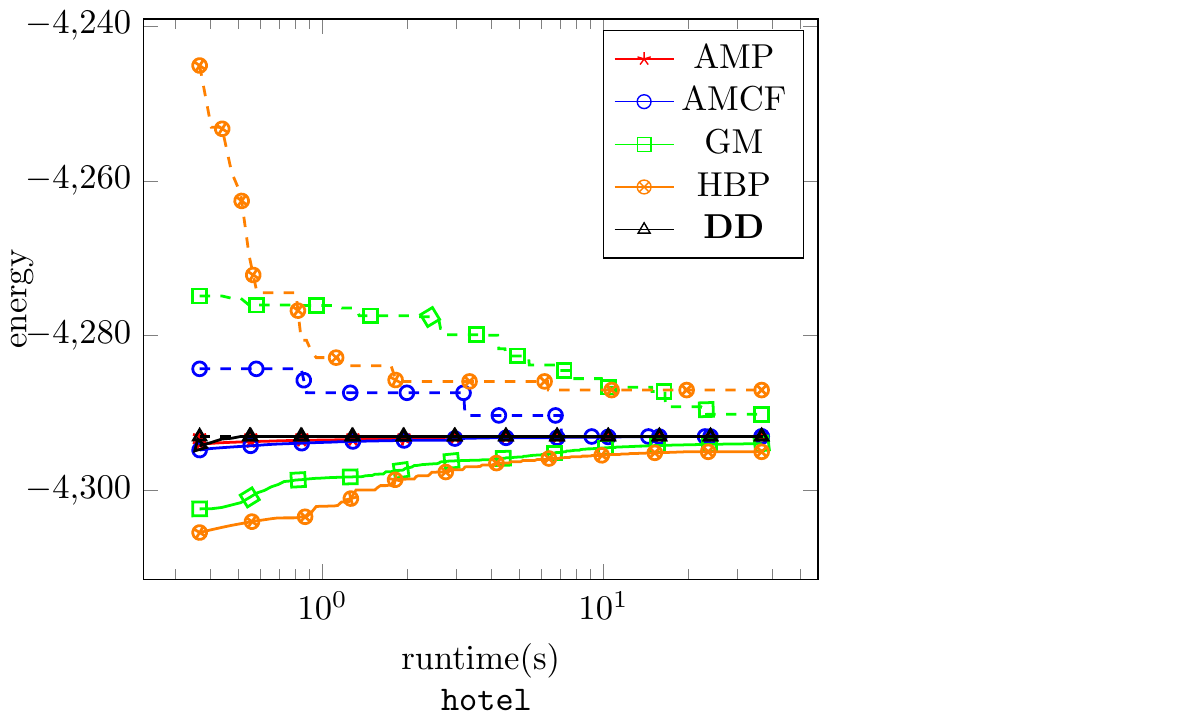}%
\end{minipage}%
\hspace*{0.4cm}%
\begin{minipage}{0.3\textwidth}%
\vspace{0pt}%
  \includegraphics[width=1.5\textwidth]{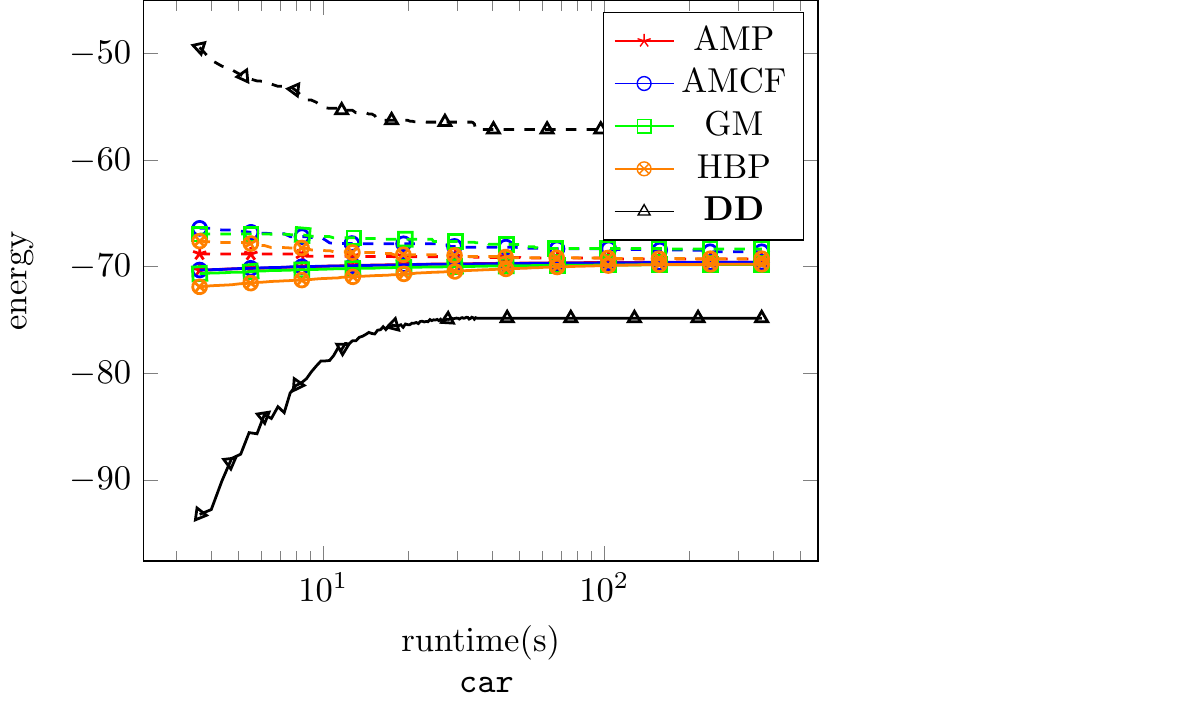}
\end{minipage}%
\\
\begin{minipage}{0.3\textwidth}%
\vspace{0pt}%
  \includegraphics[width=1.5\textwidth]{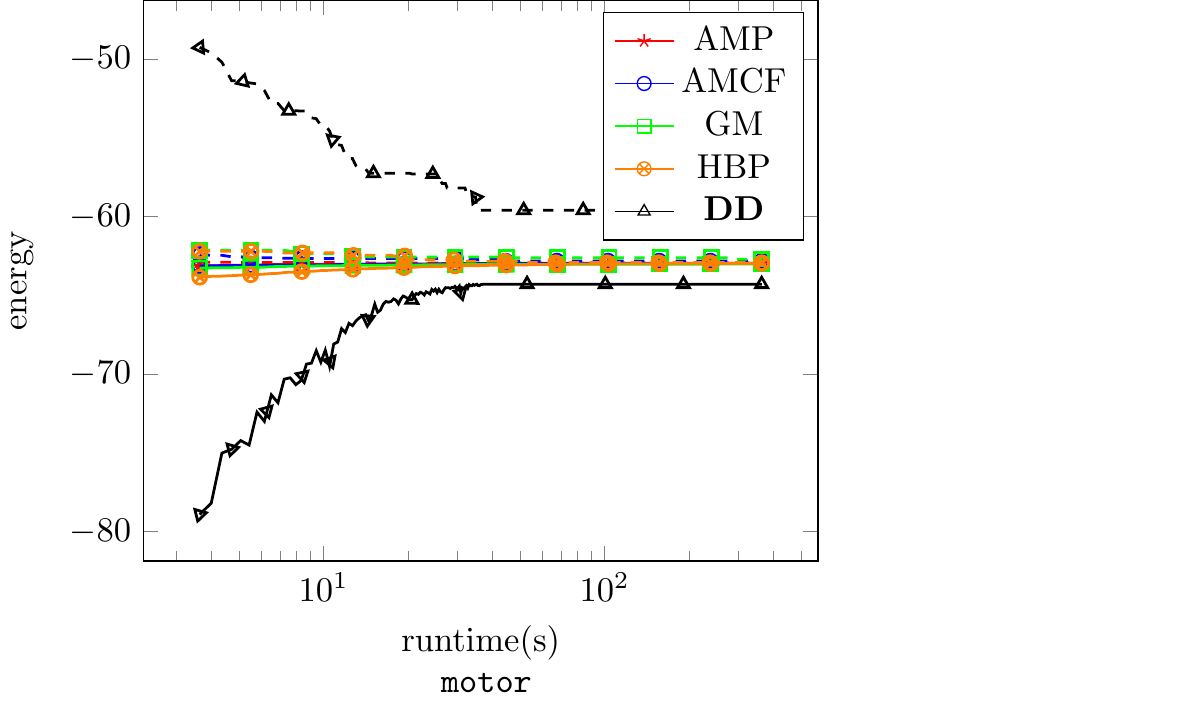}%
\end{minipage}%
\hspace*{0.8cm}%
\begin{minipage}{0.3\textwidth}%
\vspace{0pt}%
  \includegraphics[width=1.5\textwidth]{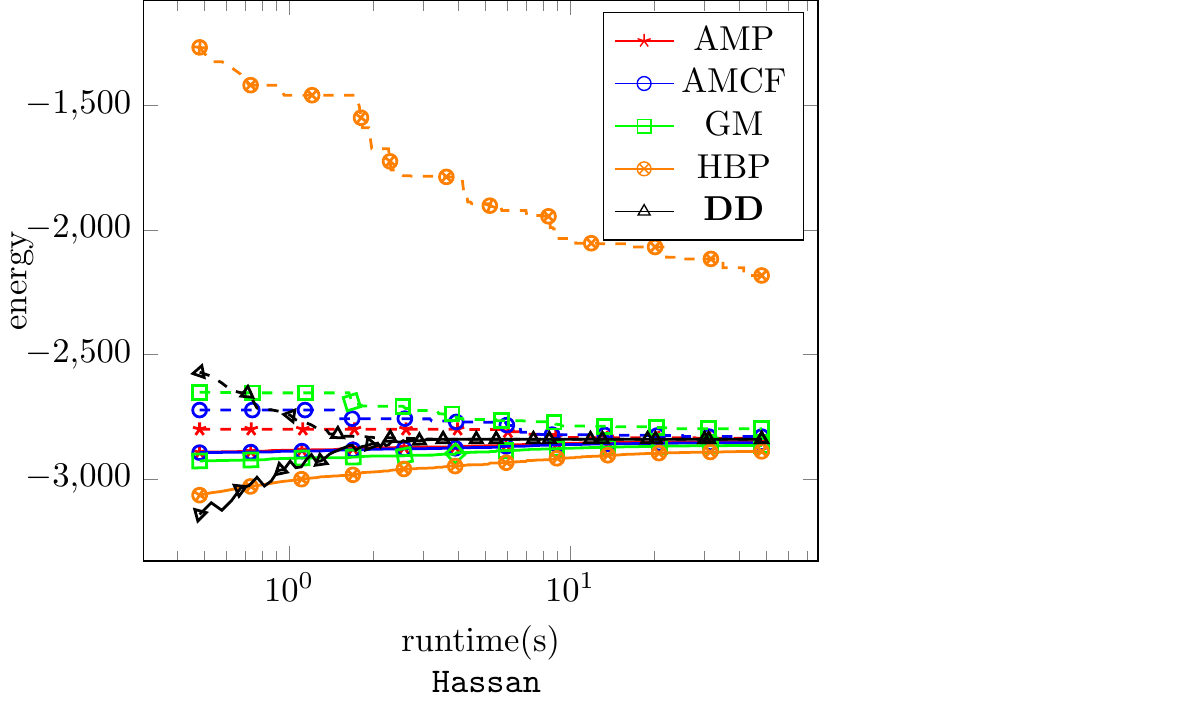}%
\end{minipage}%
\hspace*{0.4cm}%
\begin{minipage}{0.3\textwidth}%
\vspace{0pt}%
  \includegraphics[width=1.5\textwidth]{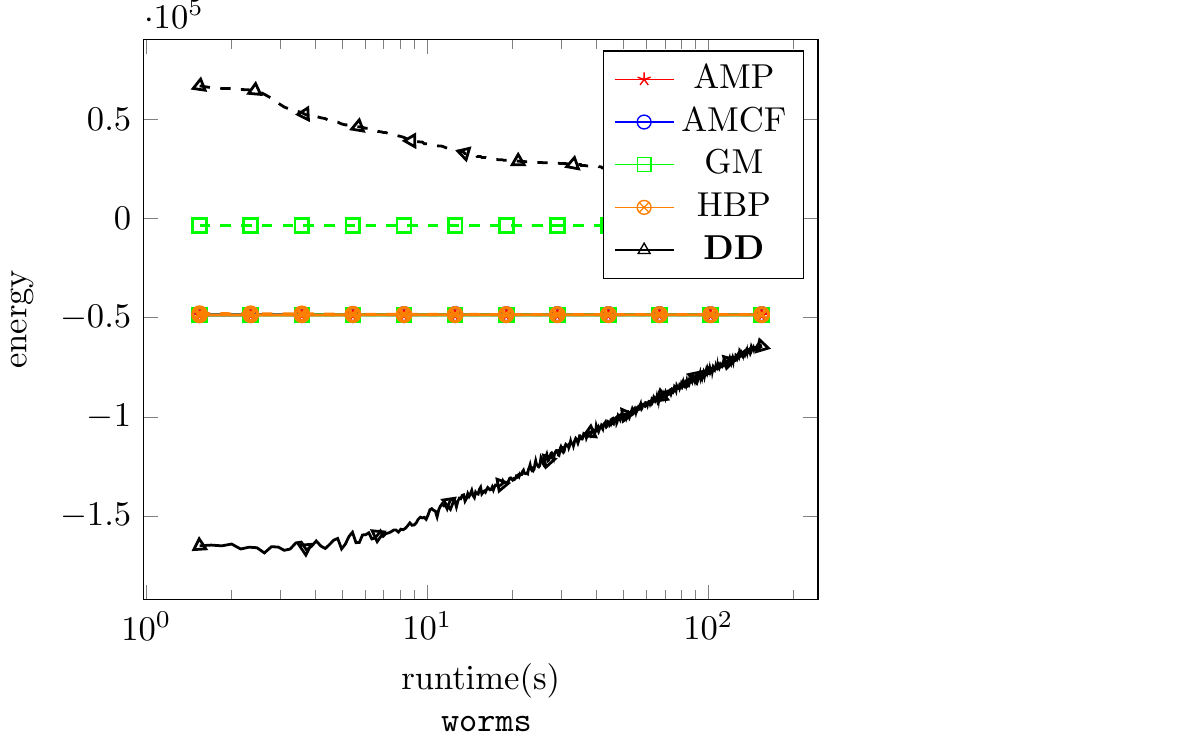}%
\end{minipage}%
  \caption{
    Runtime plots for \texttt{house}, \texttt{hotel}, \texttt{car}, \texttt{motor}, \texttt{graph flow} and \texttt{worms} datasets. 
 Continuous lines denote dual lower bounds and dashed ones primal energies. 
  Values are averaged over all instances of the dataset.
  The x-axis is logarithmic.
  }
\label{fig:DatasetRuntimePlot}
\end{figure*}

\begin{table*}
%        \small
\centering
\begin{tabular}{|cgggccccc|}
\hline
        \multicolumn{4}{|c}{ \texttt{Dataset}/\textbf{Algorithm} } & \textbf{AMP} & \textbf{AMCF} & \textbf{GM} & \textbf{HBP} & \textbf{DD} \\ \hline
        &\#I & 105 & LB & \textbf{ -4293.00 } & \textbf{ -4293.00 } & -4293.94 & -4294.98 & \textbf{ -4293.00 } \\ 
        &\#V&30& UB & \textbf{ -4293.00 } & \textbf{ -4293.00 } & -4290.12 & -4287.00 & \textbf{ -4293.00 } \\
        \multirow{-3}{*}{ \rotatebox{90}{\texttt{hotel}} }& \#L&30&  time(s) & 3.07 & 4.31 & 14.91 & 16.76 & \textbf{ 0.42 } \\ \hline
        &\#I & 105 & LB & \textbf{ -3778.13 } & \textbf{ -3778.13 } & \textbf{ -3778.13 } & \textbf{ -3778.13 } & \textbf{ -3778.13 } \\ 
        &\#V&30& UB & \textbf{ -3778.13 } & \textbf{ -3778.13 } & \textbf{ -3778.13 } & \textbf{ -3778.13 } & \textbf{ -3778.13 } \\
   \multirow{-3}{*}{ \rotatebox{90}{\texttt{house}} } & \#L&30&  time(s) & 2.81 & 3.02 & \textbf{ 1.63 } & 14.76 & 2.36 \\ \hline
        &\#I & 6& LB & -2849.96 & -2853.95 & -2865.20 & -2888.88 & \textbf{ -2840.29 } \\
\hline
\end{tabular}
\caption{
        Description of datasets together with averaged algorithm results over all instances.
        \#I means number of instances in dataset, \#V the number of nodes per instance, \#L the number of labels per instance, LB means lower bound, UB upper bound.
         \textbf{Bold numbers} indicate highest lower bound, lowest primal energy and smallest runtime.}
\label{table:DatasetResults}
\end{table*}

  {\onecolumn\footnotesize
\begin{longtable}{|llccccc|}
  \caption{Per-instance results. UB means primal solution energy, LB dual lower bound and runtime(s) the runtime of the particular algorithm in seconds.  \textbf{Bold numbers} indicate lowest primal energy, highest lower bound  and smallest runtime.}\\
\hline
  \texttt{Instance} & & \textbf{AMP} & \textbf{AMCF} & \textbf{GM} & \textbf{HBP} & \textbf{DD} \\ \hline \endhead
   \multicolumn{ 7 }{| c |}{ \textbf{ hotel } } \\ \hline
      \multirow{3}{*}{ energy\_hotel\_frame15frame22 }
      & UB & \textbf{ -4598.04 } & \textbf{ -4598.04 } & \textbf{ -4598.04 } & \textbf{ -4598.04 } & \textbf{ -4598.04 } \\
      & LB & \textbf{ -4598.04 } & \textbf{ -4598.04 } & \textbf{ -4598.04 } & \textbf{ -4598.04 } & \textbf{ -4598.04 } \\ 
      & runtime(s) & 0.15 & 0.14 & \textbf{ 0.02 } & 0.32 & 0.03 \\ \hline
   \multirow{3}{*}{ energy\_hotel\_frame15frame29 }
      & UB & \textbf{ -4540.81 } & \textbf{ -4540.81 } & \textbf{ -4540.81 } & \textbf{ -4540.81 } & \textbf{ -4540.81 } \\
      & LB
      
         &

            \textbf{ -4540.81 }

         &

            \textbf{ -4540.81 }

         &

            \textbf{ -4540.81 }

         &

            \textbf{ -4540.81 }

         &

            \textbf{ -4540.81 }

      \\ 
      & runtime(s)
      
         & 
         0.23
         
         & 
         0.14
         
         & 
         \textbf{ 0.02 }
         
         & 
         0.32
         
         & 
         0.04
         \\ \hline

      \multirow{3}{*}{ energy\_hotel\_frame15frame36 }
      & UB
      
         &

            \textbf{ -4481.33 }

         &

            \textbf{ -4481.33 }

         &

            \textbf{ -4481.33 }

         &

            \textbf{ -4481.33 }

         &

            \textbf{ -4481.33 }

      \\
\end{longtable}
}

\end{document}